\documentclass{article}

\usepackage[T1]{fontenc}
\usepackage{microtype}
\usepackage{graphicx}
\usepackage{subcaption}
\usepackage{booktabs} 
\usepackage[acronym]{glossaries}
\usepackage[inline]{enumitem}

\usepackage{tikz}
\usepackage{xparse}
\usepackage{expl3}
\usetikzlibrary{decorations.markings}
\usetikzlibrary{decorations.markings,arrows.meta,positioning,decorations.pathreplacing}
\newcommand{\TreeEndMathWrapper}{}

\newenvironment{tree}
  {%
    \ifmmode
      \def\TreeEndMathWrapper{\egroup\egroup\egroup}%
      \mathord\bgroup\vcenter\bgroup\hbox\bgroup
    \else
      \def\TreeEndMathWrapper{}%
    \fi
    \begin{tikzpicture}[
      every node/.style={
        circle,
        draw,
        fill,
        minimum size=3pt,
        inner sep=0pt,
        outer sep=0pt
      },
      line cap=round,
      baseline=(current bounding box.center)
    ]%
  }
  {%
    \end{tikzpicture}%
    \TreeEndMathWrapper
  }

\newcommand{\mynode}[1]{node[label=right:{\tiny #1}]{}}

\newcommand{\CutBarHalfWidth}{2pt}
\newcommand{\CutBarLineWidth}{0.4pt}

\tikzset{
  normaledge/.style={
    edge from parent path={
      (\tikzparentnode) -- (\tikzchildnode)
    }
  },
  cut/.style={
    edge from parent path={
      (\tikzparentnode) -- (\tikzchildnode)
      node[
        pos=.5,
        draw=none,
        fill=none,
        shape=rectangle,
        minimum size=0pt,
        inner sep=0pt,
        outer sep=0pt
      ] {%
        \tikz \draw[line width=\CutBarLineWidth]
          (-\CutBarHalfWidth,0) -- (\CutBarHalfWidth,0);
      }
    },
    every child/.style={normaledge} 
  }
}

\ExplSyntaxOn
\clist_new:N \l__tree_cut_clist

\keys_define:nn { tree }
  {
    cut .code:n    = { \clist_set:Nn \l__tree_cut_clist { #1 } },
    cut .initial:n = ,
  }

\cs_new_protected:Npn \tree_set_opts:n #1
  { \keys_set:nn { tree } { cut=, #1 } }

\cs_new_protected:Npn \tree_set_edge_style:nN #1 #2
  {
    \clist_if_in:NnTF \l__tree_cut_clist { #1 }
      { \cs_set:Npn #2 { cut } }
      { \cs_set:Npn #2 { } }
  }
\ExplSyntaxOff

\NewDocumentCommand{\TreeSetOpts}{m}{\csname tree_set_opts:n\endcsname{#1}}
\NewDocumentCommand{\TreeSetEdgeStyle}{m m}{\csname tree_set_edge_style:nN\endcsname{#1}#2}

\NewDocumentCommand{\tomato}{O{} m}{
    \begingroup
        \TreeSetOpts{#1}%
        \begin{tree}
            \draw (0,.1)\mynode{{#2}};
        \end{tree}%
    \endgroup
}

\NewDocumentCommand{\ladderTwo}{O{} m m}{
    \begingroup
        \TreeSetOpts{#1}%
        \TreeSetEdgeStyle{#2#3}{\TreeEdgeA}%
        \begin{tree}
            \node[label=right:{\tiny #2}]{}[grow'=up, level distance = 2ex, sibling distance = 2.5ex]
                child[\TreeEdgeA]{node [label=right:{\tiny #3}] {}};
        \end{tree}%
    \endgroup
}

\NewDocumentCommand{\ladderThree}{O{} m m m}{
    \begingroup
        \TreeSetOpts{#1}%
        \TreeSetEdgeStyle{#2#3}{\TreeEdgeA}%
        \TreeSetEdgeStyle{#3#4}{\TreeEdgeB}%
        \begin{tree}
            \node[label=right:{\tiny #2}]{}[grow'=up, level distance = 2ex, sibling distance = 2.5ex]
                child[\TreeEdgeA]{node [label=right:{\tiny #3}] {}
                    child[\TreeEdgeB]{node [label=right:{\tiny #4}]{}}
                };
        \end{tree}%
    \endgroup
}

\NewDocumentCommand{\cherry}{O{} m m m}{
    \begingroup
        \TreeSetOpts{#1}%
        \TreeSetEdgeStyle{#2#3}{\TreeEdgeA}%
        \TreeSetEdgeStyle{#2#4}{\TreeEdgeB}%
        \begin{tree}
            \node[label=right:{\tiny #2}]{}[grow'=up, level distance = 2ex, sibling distance = 2.5ex]
                child[\TreeEdgeA]{node [label=right:{\tiny #3}] {}}
                child[\TreeEdgeB]{node [label=right:{\tiny #4}] {}};
        \end{tree}%
    \endgroup
}

\NewDocumentCommand{\rightTree}{O{} m m m m}{
    \begingroup
        \TreeSetOpts{#1}%
        \TreeSetEdgeStyle{#2#3}{\TreeEdgeA}%
        \TreeSetEdgeStyle{#2#4}{\TreeEdgeB}%
        \TreeSetEdgeStyle{#4#5}{\TreeEdgeC}%
        \begin{tree}
            \node[label=right:{\tiny #2}]{}[grow'=up, level distance = 2ex, sibling distance = 2.5ex]
                child[\TreeEdgeA]{node [label=right:{\tiny #3}] {}}
                child[\TreeEdgeB]{node [label=right:{\tiny #4}] {}
                    child[\TreeEdgeC]{node [label=right:{\tiny #5}] {}}
                };
        \end{tree}%
    \endgroup
}

\NewDocumentCommand{\leftTree}{O{} m m m m}{
    \begingroup
        \TreeSetOpts{#1}%
        \TreeSetEdgeStyle{#2#3}{\TreeEdgeA}%
        \TreeSetEdgeStyle{#2#4}{\TreeEdgeB}%
        \TreeSetEdgeStyle{#3#5}{\TreeEdgeC}%
        \begin{tree}
            \node[label=right:{\tiny #2}]{}[grow'=up, level distance = 2ex, sibling distance = 2.5ex]
                child[\TreeEdgeA]{node [label=right:{\tiny #3}] {}
                    child[\TreeEdgeC]{node [label=right:{\tiny #5}] {}}
                }
                child[\TreeEdgeB]{node [label=right:{\tiny #4}] {}};
        \end{tree}%
    \endgroup
}

\NewDocumentCommand{\Ytree}{O{} m m m m}{
    \begingroup
        \TreeSetOpts{#1}%
        \TreeSetEdgeStyle{#2#3}{\TreeEdgeA}%
        \TreeSetEdgeStyle{#3#4}{\TreeEdgeB}%
        \TreeSetEdgeStyle{#3#5}{\TreeEdgeC}%
        \begin{tree}
            \node[label=right:{\tiny #2}]{}[grow'=up, level distance = 2ex, sibling distance = 2.5ex]
                child[\TreeEdgeA]{node [label=right:{\tiny #3}] {}
                    child[\TreeEdgeB]{node [label=right:{\tiny #4}] {}}
                    child[\TreeEdgeC]{node [label=right:{\tiny #5}] {}}
                };
        \end{tree}%
    \endgroup
}

\NewDocumentCommand{\bushThree}{O{} m m m m}{
    \begingroup
        \TreeSetOpts{#1}%
        \TreeSetEdgeStyle{#2#3}{\TreeEdgeA}%
        \TreeSetEdgeStyle{#2#4}{\TreeEdgeB}%
        \TreeSetEdgeStyle{#2#5}{\TreeEdgeC}%
        \begin{tree}
            \node[label=right:{\tiny #2}]{}[grow'=up, level distance = 2ex, sibling distance = 2.5ex]
                child[\TreeEdgeA]{node [label=right:{\tiny #3}] {}}
                child[\TreeEdgeB]{node [label=right:{\tiny #4}] {}}
                child[\TreeEdgeC]{node [label=right:{\tiny #5}] {}};
        \end{tree}%
    \endgroup
}

\usepackage{tikz}
\usepackage{tikz-cd}
\usepackage{multirow}
\usepackage{xcolor}
\PassOptionsToPackage{hyphens}{url}
\usepackage{xurl}

\usepackage[
  colorlinks=true,
  linkcolor=red!70!black,
  linktocpage=false,
  citebordercolor=blue!70!black,
  citecolor=blue!70!black,
  anchorcolor=blue!70!black,
  bookmarks=true,
  bookmarksopen=true,
  bookmarksnumbered=true,
  pdfpagemode=UseOutlines
]{hyperref}

\usepackage{bookmark}

\DeclareMathOperator{\hopfshuf}{\mathcal{H}_{\shuffle}^{\circ}}

\DeclareMathOperator{\hopfmkw}{\mathcal{H}_\mathrm{MKW}^{\circ}}

\DeclareMathOperator{\hopfgl}{\mathcal{H}_\mathrm{GL}}

\newacronym{sota}{SOTA}{state-of-the-art}
\newacronym{ad}{AD}{auto-differentiation}

\newacronym{lj}{LJ}{Lennard-Jones}

\newacronym{fbm}{fBM}{fractional brownian motion}
\newacronym{rl}{RL}{Riemann-Liouville}
\newacronym{ou}{OU}{Ornstein-Uhlenbeck}

\newacronym{rbergomi}{rBergomi}{rough Bergomi}

\newacronym{rge}{RGE}{rotational geodesic error}
\newacronym{mse}{MSE}{mean squared error}
\newacronym{mmd}{MMD}{maximum mean discrepancy}

\newacronym{ode}{ODE}{ordinary differential equation}
\newacronym{pde}{PDE}{partial differential equation}
\newacronym{sde}{SDE}{stochastic differential equation}
\newacronym{cde}{CDE}{controlled differential equation}
\newacronym{yde}{YDE}{Young differential equation}
\newacronym{rde}{RDE}{rough differential equation}
\newacronym{rsde}{RSDE}{rough stochastic differential equation}
\newacronym{spde}{SPDE}{stochastic partial differential equation}

\newacronym{rk}{RK}{Runge--Kutta}
\newacronym{rkmk}{RKMK}{Runge--Kutta--Munthe--Kaas}
\newacronym{cg}{CG}{Crouch--Grossman}
\newacronym{cf}{CF}{Commutator free}
\newacronym{ees}{EES}{Explicit and Effectively Symmetric}

\newacronym{sg}{SG}{Savitzky--Golay}

\newacronym{node}{NODE}{neural ODE}
\newacronym{mnode}{M-NODE}{manifold neural ODE}
\newacronym{ncde}{NCDE}{neural controlled differential equation}
\newacronym{nrde}{NRDE}{neural rough differential equation}
\newacronym{bnrde}{B-NRDE}{branched neural rough differential equation}

\newacronym{nspde}{NSPDE}{neural stochastic partial differential equation}
\newacronym{dlrnet}{DLR-Net}{Deep Latent Regularity Net}
\newacronym{nors}{NORS}{Neural Operator with Regularity Structure}

\newacronym{gl}{GL}{Grossman--Larson}
\newacronym{bck}{BCK}{Butcher--Connes--Kreimer}
\newacronym{mkw}{MKW}{Munthe--Kaas--Wright}

\newacronym{smiles}{SMILES}{Simplified Molecular-input Line-entry System}
\newacronym{dft}{DFT}{density functional theory}
\newacronym{md}{MD}{molecular dynamics}

\newacronym{spdgraph}{SPD}{shorest path distance}

\newacronym{spdman}{SPD}{symmetric positive definite}
\newacronym{aiman}{AI}{affine invariant}
\newacronym{bw}{BW}{Bures-Wasserstein}

\newacronym{mlp}{MLP}{multilayer perceptron}
\newacronym{nlp}{NLP}{natural language processing}
\newacronym{llm}{LLM}{large language models}
\newacronym{ffn}{FFN}{feedforward neural network}
\newacronym{cnn}{CNN}{convolutional neural network}
\newacronym{mpnn}{MPNN}{message passing neural network}

\newacronym{mt}{MT}{multitask}
\newacronym{zs}{ZS}{zero-shot}

\newacronym{relu}{ReLU}{rectified linear unit}
\newacronym{gelu}{GELU}{Gaussian error linear unit}
\newacronym{silu}{SiLU}{sigmoid linear unit}

\newacronym{gnn}{GNN}{graph neural network}
\newacronym{gtn}{GTN}{graph transformer neural network}
\newacronym{egnn}{EGNN}{Equivariant Graph Neural Network}
\newacronym{gat}{GAT}{graph attention network}
\newacronym{gatv2}{GATv2}{graph attention network v2}

\newacronym{mha}{MHA}{multi-headed scaled dot-product attention}
\newacronym{rope}{RoPE}{Rotary Position Embedding}
\newacronym{trope}{T-RoPE}{Temporal Rotary Position Embedding}
\newacronym{nope}{NoPE}{no positional encoding}

\newacronym{rwpe}{RWPE}{random walk positional encoding}
\newacronym{rrwp}{RRWP}{relative random walk probabilities}

\newacronym{no}{NO}{neural operator}
\newacronym{atoms}{ATOM}{Atomistic Transformer Operator for Molecules}
\newacronym{fno}{FNO}{Fourier Neural Operator}
\newacronym{egno}{EGNO}{Equivariant Graph Neural Operator}
\newacronym{gnot}{GNOT}{General Neural Operator Transformer}
\newacronym{hnca}{HNCA}{Heterogenous Normalised Cross-Attention}

\newacronym{ks}{KS}{Kolmogorov-Smirnov}

\newacronym{omd}{OMD}{Oxford Multimotion Dataset}
 
\glsdisablehyper

\usepackage[accepted]{icml2026}

\usepackage{amsmath}
\usepackage{amssymb}
\usepackage{mathtools}
\usepackage{amsthm}
\usepackage{dirtytalk}
\usepackage{iftex}
\ifXeTeX
  \newcommand{\paperShuffleGlyph}[1]{%
    \tikz[baseline=-0.15ex,scale=#1]{%
      \draw[line width=0.4pt] (0,0.15)--(0,0)--(0.24,0)--(0.24,0.15)
        (0.12,0)--(0.12,0.15);}}
  \DeclareRobustCommand{\shuffle}{\mathbin{\mathchoice
    {\paperShuffleGlyph{1}}{\paperShuffleGlyph{1}}
    {\paperShuffleGlyph{0.7}}{\paperShuffleGlyph{0.5}}}}
\else
  \usepackage{shuffle}
\fi
\usepackage{tabularray}
\usepackage{stmaryrd}  

\usepackage{xparse}

\makeatletter
\let\oldsection\section
\newcounter{autopdfbookmarksection}

\RenewDocumentCommand{\section}{s o m}{%
  \stepcounter{autopdfbookmarksection}%
  \phantomsection
  \IfNoValueTF{#2}
    {\pdfbookmark[1]{#3}{bookmark:section:\theautopdfbookmarksection}}
    {\pdfbookmark[1]{#2}{bookmark:section:\theautopdfbookmarksection}}%
  \IfBooleanTF{#1}
    {\oldsection*{#3}}
    {%
      \IfNoValueTF{#2}
        {\oldsection{#3}}
        {\oldsection[#2]{#3}}%
    }%
}
\makeatother

\makeatletter
\let\oldsubsection\subsection
\newcounter{autopdfbookmarksubsection}

\RenewDocumentCommand{\subsection}{s o m}{%
  \stepcounter{autopdfbookmarksubsection}%
  \phantomsection
  \IfNoValueTF{#2}
    {\pdfbookmark[2]{#3}{bookmark:subsection:\theautopdfbookmarksubsection}}
    {\pdfbookmark[2]{#2}{bookmark:subsection:\theautopdfbookmarksubsection}}%
  \IfBooleanTF{#1}
    {\oldsubsection*{#3}}
    {%
      \IfNoValueTF{#2}
        {\oldsubsection{#3}}
        {\oldsubsection[#2]{#3}}%
    }%
}
\makeatother

\usepackage[capitalize,noabbrev,nameinlink]{cleveref}
\crefname{definition}{Definition}{Definitions}
\crefname{theorem}{Theorem}{Theorem}
\AddToHook{cmd/appendix/before}{%
  \crefalias{section}{appendix}%
  \crefalias{subsection}{appendix}%
}
\usepackage{import}
\usepackage{xifthen}
\usepackage{pdfpages}
\usepackage{transparent}

\theoremstyle{plain}
\newtheorem{theorem}{Theorem}[section]
\newtheorem{proposition}[theorem]{Proposition}

\theoremstyle{definition}
\newtheorem{definition}[theorem]{Definition}

\theoremstyle{remark}
\newtheorem{remark}[theorem]{Remark}
\theoremstyle{definition}
\newtheorem{example}[theorem]{Example}

\usepackage[textsize=tiny]{todonotes}

\usepackage{pifont}
\newcommand{\cmark}{\ding{51}}%
\newcommand{\xmark}{\ding{55}}%

\icmltitlerunning{Branched Neural Rough Differential Equations}

\begin{document}

\def\TreeOneTwo#1#2#3{\tikz[baseline = -0.3ex]{
		\draw[fill] (0,0) circle [radius=0.06];
		\draw[fill] (-0.2,0.35) circle [radius=0.06];
		\draw[fill] (0.2,0.35) circle [radius=0.06];
		\draw (0,0) -- (-0.2,0.35);
		\draw (0,0) -- (0.2,0.35);
		\node[right] at (0,0) {$\scriptstyle{#1}$};
		\node[above] at (-0.2,0.35) {$\scriptstyle{#2}$};
		\node[above] at (0.2,0.35) {$\scriptstyle{#3}$};}}

\twocolumn[
  \icmltitle{\texorpdfstring{Learning Manifold and It\^o Dynamics with\\Branched$^{\mkern-9mu\TreeOneTwo{}{}{}}$ Neural Rough Differential Equations}{Learning Manifold and It\^o Dynamics with Branched Neural Rough Differential Equations}}



  \icmlsetsymbol{equal}{*}

  \begin{icmlauthorlist}
    \icmlauthor{Luke Thompson}{equal,usyd}
    \icmlauthor{Dai Shi}{equal,cambridge}
    \icmlauthor{Lequan Lin}{usyd}
    \icmlauthor{Junbin Gao}{usyd}
    \icmlauthor{Andi Han}{usyd}
  \end{icmlauthorlist}
  \icmlaffiliation{usyd}{University of Sydney, Australia}
  \icmlaffiliation{cambridge}{Cambridge University, United Kingdom}

  \icmlcorrespondingauthor{Luke Thompson}{luke-a-thompson@outlook.com}
  \icmlcorrespondingauthor{Andi Han}{andi.han@sydney.edu.au}

  \icmlkeywords{Dynamical systems, Manifold learning, Differential equations}

  \vskip 0.3in
    ]



\printAffiliationsAndNotice{}  


\begin{abstract}

Neural rough differential equations (NRDEs) stay accurate under irregular sampling while taking far fewer integration steps than standard neural differential equations, summarising a finely sampled driver by its log-signature and advancing the hidden state over coarse intervals using the log-ODE method. This efficiency rests on the shuffle algebra, the algebraic counterpart of Stratonovich calculus. This reliance means NRDEs cannot expose the quadratic-variation terms It\^o dynamics require, nor the ordered covariant derivatives that govern It\^o flows on connection-equipped manifolds. Ameliorating this, we introduce Branched Neural Rough Differential Equations (B-NRDEs), a Hopf-algebraic framework that recasts the NRDE log-ODE step as geometric numerical integration on the state-space manifold, matching the driving algebra to the governing calculus: Grossman--Larson rooted trees for Euclidean It\^o dynamics, the Munthe--Kaas--Wright dual for ordered covariant derivatives on manifolds, and the tensor algebra in the classical Stratonovich case. This yields intrinsic coarse-step dynamics that exactly preserve manifold constraints. Finally, we introduce a branched signature-kernel objective to enable It\^o-consistent law matching by making quadratic variation terms visible during training. On rough Bergomi volatility, sim-to-real $\mathrm{SO}(3)$ forecasting, and SPD covariance dynamics, B-NRDEs offer a unified, effective approach to stochastic and manifold-valued dynamics beyond the Euclidean-Stratonovich setting.
\end{abstract}

\section{Introduction}
\glsresetall
Learning continuous-time dynamics from time series is a foundational problem in machine learning, with applications ranging from robotics \cite{duong_hamiltonian-based_2021, chee_knode-mpc_2022} and molecular dynamics \cite{huang_generalizing_2023, schreiner_implicit_2023} to quantitative finance \cite{gierjatowicz_robust_2020, issa_non-adversarial_2023}. \Glspl{ncde} \cite{kidger_neural_2020} approach this task by parameterising the vector field of a controlled differential equation with a neural network. To improve scalability and robustness to sampling rates, \glspl{nrde} \cite{morrill_neural_2021} lift the control path to its log-signature, a series of iterated integrals, and solve the resulting system via the log-ODE method at coarser discretisations \cite{castell_ordinary_1996}.

\begin{figure}[t]
    \centering
    \includegraphics[width=\linewidth]{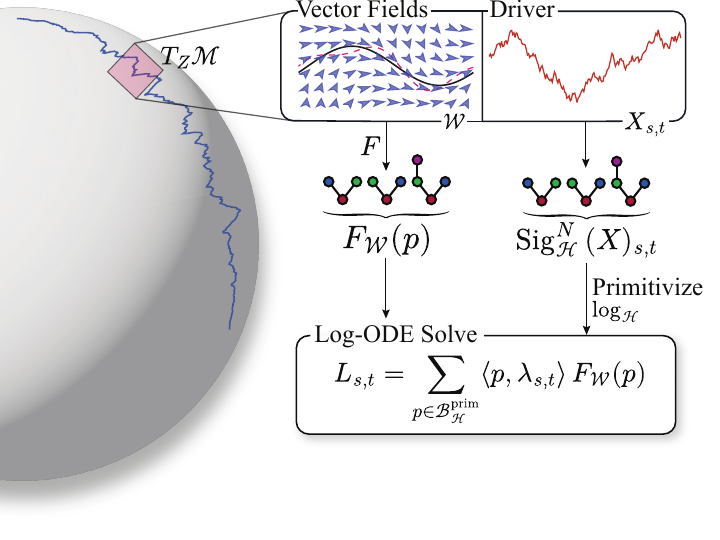}
    \vspace{-2em}
    \caption{Illustration of the log-ODE method for \glspl{bnrde}. The control segment $X_{s,t}$ is lifted to $\mathbf{X}^\mathcal{H}_{s,t} \in \mathcal{H}^{(\le N)}$ for the chosen Hopf algebra $\mathcal{H} \in \{\hopfshuf, \hopfgl, \hopfmkw\}$, then transformed to a log-signature $\lambda_{s,t}$ in primitive coordinates $p \in \mathcal{B}_\mathcal{H}^\mathrm{prim}$. The tangent vector fields $\mathcal{W}$ are lifted through the pseudo bialgebra map $F_{\mathcal{W}}$ to matching primitive-indexed vector fields on the state manifold. These are combined into the log-ODE vector field $L_{s,t}$, whose ODE flow gives the local update from $s$ to $t$.}    \label{fig:main_fig}
    \vspace{-2em}
\end{figure}

However, geometric \glspl{nrde} are driven by geometric signatures whose coordinates satisfy the shuffle product identity. This is appropriate for Stratonovich integration, which preserves the usual chain rule, but not for It\^o integration. It\^o's product rule generates quadratic-variation corrections, so products of It\^o iterated integrals contain second-order terms that are not represented by the original tensor coordinates. This limits the use of geometric signatures in settings where adapted causal It\^o modelling assumptions are required. On manifolds, the obstruction is sharper. It\^o integration is defined relative to a connection, and its expansion involves higher covariant derivatives. Since these operators do not generally commute, the signature algebra must preserve their ordered, branched composition rather than identify them through shuffle relations.

To address these limitations, we introduce \glspl{bnrde}, a unified framework for learning dynamics using the log-ODE method under It\^o integration and over manifolds. By lifting the control path to a Hopf algebra of \textit{rooted trees} --- the \acrlong{gl} algebra for Euclidean It\^o calculus or the graded dual of the \acrlong{mkw} algebra for manifolds, we derive a generalised \gls{nrde} that strictly respects the geometric and causal constraints of the domain. See \cref{tab:hopf_summary} for comparisons of Hopf algebras and \cref{fig:main_fig} for an illustration of the proposed framework.
Our \textbf{main contributions} are as follows:
\begin{enumerate}[leftmargin=0.2in]

    \item We model path signatures in the \acrlong{gl} Hopf algebra \(\mathcal{H}_{\text{GL}}\), whose rooted-tree basis represents It\^o-type iterated integrals. We then introduce the first neural dynamics model utilising the graded dual of the \acrlong{mkw} Hopf algebra $\hopfmkw$ of ordered forests of planar rooted trees to strictly enforce manifold constraints for It\^o-type dynamical systems.
    
    \item We unify the above approaches via \textit{pseudo bialgebra maps} \cite{kern_flow_2023}, which convert Hopf algebraic structures to learned vector fields and differential operators on the target manifold. This yields a general log-ODE formulation for Stratonovich and It\^o dynamics on connection-equipped manifolds. In this work, we instantiate the numerical solver on homogeneous spaces, with tangent vectors represented in frame or Lie-algebra coordinates and integrated through a group action.

    \item We enable It\^{o}-consistent neural dynamics training both in Euclidean space and on manifolds by developing a branched signature-kernel objective.
    
    \item We release \textit{\href{https://github.com/luke-a-thompson/Roughrax}{Roughrax}}, a JAX package for branched rough paths, and autodifferentiable numerical solution for \glspl{rde} over manifolds.
\end{enumerate}

\section{Preliminaries}
\label{sec:math_prelims}
In this section, we first review the standard formulation of Neural CDEs and RDEs to establish the Euclidean baseline (\Cref{sec:ncde_nrde}). We then discuss the machinery required to extend the Log-ODE framework to non-Euclidean and It\^o settings. We first introduce rough path Hopf algebras (\Cref{sec:prelim_rough_path_hopf_algebras}), then present pseudo bialgebra maps that send algebra elements to differential operators on $\mathcal M$ (\Cref{sec:prelim_vf_pseudo_bialgebra}). 

\subsection{Neural Controlled and Rough Differential Equations}
\label{sec:ncde_nrde}

Let $\big((t_0, x_0), \dots, (t_n, x_n) \big)$ be a potentially irregular time series with observations $x_i \in \mathbb{R}^d$. We construct a continuous interpolation $X \colon [t_0, t_n] \to \mathbb{R}^d$ such that $X_{t_i} = x_i$.

\textbf{Neural Controlled Differential Equations.}
A \gls{ncde} \cite{kidger_neural_2020} evolves a hidden state $h_t \in \mathbb{R}^u$ driven by the path $X$. Let $\xi_\phi: \mathbb{R}^d \to \mathbb{R}^u$, $\ell_\psi: \mathbb{R}^u \to \mathbb{R}^w$  be an encoder and output network respectively, and $g_\theta: \mathbb{R}^u \to \mathbb{R}^{u \times d}$ be a vector field parametrisation. The forward pass is defined by the Riemann-Stieltjes integral:
\begin{align}
    h_t = h_{t_0} + \int_{t_0}^{t} g_\theta(h_s) \, \mathrm{d}X_s,  \label{eq:ncde}
\end{align}
where initial hidden state $h_{t_0} = \xi_\phi(X_{t_0})$, and output $y_t = \ell_\psi(h_t)$, $g_\theta(h_s)\,\mathrm{d}X_s$ denotes matrix-vector multiplication. This formulation allows the latent state to react continuously to incoming data.

\textbf{Neural Rough Differential Equations.}
While highly expressive, \glspl{ncde} are expensive for long sequences. \Glspl{nrde} \cite{morrill_neural_2021} reduce this cost by applying the log-ODE method: over each window \(I_j=[t_j,t_{j+1}]\), the path \(X\) is summarised by its log-signature \(\lambda_j\). This log-signature determines an autonomous vector field, denoted \(\tilde g_{\theta,\lambda_j}\), obtained by applying the log-ODE lift to the base neural vector field. The hidden state is then advanced by the \gls{ode}
\begin{equation}
    \label{eq:nrde}
    h_t
    =
    h_{t_j}
    +
    \int_{t_j}^{t}
    \tilde g_{\theta,\lambda_j}(h_s)\,\mathrm ds,
    \qquad t\in I_j .
\end{equation}
On physical time, $\tilde g_{\theta,\lambda_j}=F_{W_\theta}(\lambda_j)/(t_{j+1}-t_j)$. Here, \(\lambda_j\) sets the coefficients of the ODE solved on \(I_j\), rather than serving as a new path differential, yielding a higher-order approximation that permits larger step sizes.

\subsection{Rough Path Hopf Algebras}
\label{sec:prelim_rough_path_hopf_algebras}
Hopf algebras organise iterated-integral combinatorics and the product identities needed for computation. We work in the graded duals of the shuffle, BCK and MKW coordinate algebras, denoted $\hopfshuf$, $\hopfgl$ and $\hopfmkw$, respectively. Their completed group-like elements are characters of the coordinate algebras (\cref{app:duality_dictionary}). \cref{tab:hopf_summary} summarises the regimes supported by each Hopf algebra.

\begin{definition}[Hopf algebra and primitives]\label[definition]{def:hopf_algebra}
Let \(\mathcal{H}=(H,\star,\eta,\Delta,\varepsilon,S)\) be a Hopf algebra on a vector space \(H\) with product \(\star:H\otimes H\to H\) and coproduct \(\Delta:H\to H\otimes H\); we denote a basis of \(H\) by \(\mathcal{B}_{\mathcal H}\), and write \(\mathrm{Prim}(\mathcal H)\coloneqq \{p\in H:\Delta p=p\otimes 1+1\otimes p\}\) with a chosen primitive basis \(\mathcal{B}_{\mathcal H}^{\mathrm{prim}}\subset \mathrm{Prim}(\mathcal H)\). We suppress \(\eta\), \(\varepsilon\) and \(S\) below, since only \(\star\) and \(\Delta\) are immediately relevant to our constructions.
\end{definition}

\begin{table}[H]
\centering
\caption{Integration regimes in the cited neural differential-equation formulations, organised by their working Hopf algebra.}
\label{tab:hopf_summary}
\setlength{\tabcolsep}{6pt}
\renewcommand{\arraystretch}{1.15}
\resizebox{\linewidth}{!}{
\begin{tblr}{
  colspec = {l c c c},
  row{1} = {font=\bfseries},
  cell{1}{1} = {l},
  cell{1}{2} = {c},
  cell{1}{3} = {c},
  cell{1}{4} = {c},
  hline{1,Z} = {-}{0.08em},
  hline{2}   = {-}{0.06em},
}
Hopf algebra  & Stratonovich & Euclidean It\^o & Manifold It\^o \\
$\hopfshuf$  \citep{morrill_neural_2021} & \cmark & \xmark & \xmark \\
$\hopfgl$ (this work)                     & \xmark & \cmark & \xmark \\
$\hopfmkw$  (this work)                   & \xmark & \cmark & \cmark \\
\end{tblr}
}
\end{table}
\vspace{-0.5em}
\paragraph{Stratonovich (geometric) signatures and $\hopfshuf$.}
For a $d$-dimensional control path $X=(X^{(1)},\dots, X^{(d)})$, the truncated geometric signature collects its iterated Stratonovich integrals in word coordinates indexed by $e_{i_1}\otimes\cdots\otimes e_{i_m}$. We use the standard tensor-algebra convention in which signatures are group-like, and log-signatures are computed using the tensor product; the shuffle product appears dually as the product identity satisfied by signature coordinate functions \cite{kidger_neural_2020,morrill_neural_2021}. The defining feature of this setting is the Stratonovich product rule. At level two, we have
\begin{equation}
    \label{eq:shuffle_strat}
    X_{s,t}^{(i)} X_{s,t}^{(j)}
    =
    \int_s^t X_{s,u}^{(i)} \circ \mathrm{d}X_u^{(j)}
    +
    \int_s^t X_{s,u}^{(j)} \circ \mathrm{d}X_u^{(i)}.
\end{equation}
Algebraically, \eqref{eq:shuffle_strat} is encoded by \(e_i \shuffle e_j = e_i \otimes e_j + e_j \otimes e_i\). Thus \(\hopfshuf\) naturally represents geometric (Stratonovich) rough paths: iterated integrals are indexed by tensor coordinates and obey shuffle identities. Since Stratonovich integration is coordinate-free, the geometric signature remains appropriate for Stratonovich dynamics on manifolds.

\paragraph{It\^o (branched) signatures and $\hopfgl$.}
Under It\^o integration, the product rule includes a quadratic-variation term:
\begin{equation}
\label{eq:gl_ito}
    \begin{aligned}
    X_{s,t}^{(i)} X_{s,t}^{(j)}
    = \int_s^t X_{s,u}^{(i)} \mathrm{d}X_u^{(j)}
     + \int_s^t &X_{s,u}^{(j)} \mathrm{d}X_u^{(i)}\\[-0.5em]
    &+ \langle X^{(i)}, X^{(j)} \rangle_{s,t}.
    \end{aligned}
\end{equation}
Over the original $d$-dimensional word coordinates, the correction term $\langle X^{(i)}, X^{(j)}\rangle$ is not exposed as an independent level-two driver coordinate. It can be represented only indirectly, for example, by augmenting the path through a lead--lag lift. Hence, the unaugmented shuffle representation is not the natural signature space for It\^o modelling.
We therefore work in the Hopf algebra of \citet{grossman_hopf-algebraic_1989}, $\hopfgl$, on rooted trees, which underpins branched rough paths \citep{gubinelli_ramification_2010}. Here, trees index non-geometric iterated-integral coordinates. At degree two, quadratic variation is a linear combination of forest and tree coordinates, or equivalently minus twice the symmetric part of the tree log-coordinates; see \cref{app:depth_two_ito}.
\vspace{-0.5em}
\paragraph{It\^o (branched) signatures over manifolds and $\hopfmkw$.} For dynamics on a manifold $\mathcal M$ with connection $\nabla$, the local flow involves \emph{ordered} compositions of covariant derivatives: in general $\nabla_U\nabla_V \neq \nabla_V\nabla_U$, and the induced curvature terms depend on the order in which directions are applied. This order-sensitivity is represented by the graded dual of the \citet{munthe-kaas_hopf_2006} Hopf algebra, $\hopfmkw$, on ordered forests of planar rooted trees. The left-to-right child order indexes ordered covariant derivatives. In contrast, non-planar rooted trees symmetrise child order and collapse the order-sensitive flow terms. See \Cref{appx:hopf_algebras_of_rough_paths} for further algebraic details. We show a visual representation of the Stratonovich and It\^o conventions in \cref{fig:riemann-stieltjes}.

\begin{figure}[h]
\centering
\begin{subfigure}[t]{0.49\linewidth}
    \centering
    \includegraphics[]{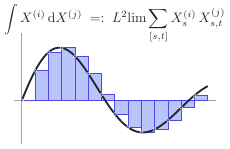}
    \caption{It\^o convention.}
    \label{fig:rs-ito}
\end{subfigure}
\begin{subfigure}[t]{0.49\linewidth}
    \centering
    \includegraphics[]{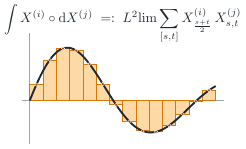}
    \caption{Stratonovich convention.}
    \label{fig:rs-stratonovich}
\end{subfigure}
\hfill

\caption{Riemann--Stieltjes approximations under different evaluation conventions. Redrawn from \citet{bellingeri_branched_2024}.}
\label{fig:riemann-stieltjes}
\vspace{-1em}
\end{figure}

\subsection{Vector Fields and Pseudo Bialgebra Maps}
\label{sec:prelim_vf_pseudo_bialgebra}

Pseudo-bialgebra maps \cite{kern_flow_2023} identify how signature coordinates act on functions through vector fields. Let $\mathcal M$ be a smooth manifold, let $\Gamma(T\mathcal M)$ denote its smooth vector fields, and let $\mathrm{Diff}(\mathcal M)$ be the algebra of linear differential operators on $C^\infty(\mathcal M)$. Given a connection $\nabla$ when covariant derivatives are required, a pseudo-bialgebra map sends each word or tree basis element of the chosen Hopf algebra to the corresponding differential operator on $\mathcal M$. Thus, the signature expansion and the induced flow expansion are evaluated in the same word- or tree-indexed coordinates.

\begin{definition}[pseudo bialgebra map \texorpdfstring{\citep[Def.~4.1]{kern_flow_2023}}{}]\label[definition]{def:pseudo_bialgebra_map}
    A linear map $F\colon \mathcal{H} \to \mathrm{Diff}(\mathcal{M})$ is a pseudo bialgebra map if $F(1)=\mathrm{Id}$ and
    \[
        F(\tau \star \sigma)=F(\tau)\circ F(\sigma)
    \]
    and, for all $\tau\in\mathcal H$ and $\phi,\psi\in C^\infty(\mathcal M)$,
    \(
        m\circ(F\otimes F)(\Delta\tau)(\phi\otimes\psi)
        =
        F(\tau)(\phi\psi),
    \)
    where $m(a\otimes b)=ab$ denotes pointwise multiplication of functions.
\end{definition}

Given a collection of smooth vector fields $\mathcal W = (W^{(1)}, \dots, W^{(d)})$ on $\mathcal{M}$, the elementary differential associated with a rooted tree is defined recursively. For the single-node tree \(F_{\mathcal W}(\tomato{i})=W^{(i)}\). For a rooted tree $\tau = [\tau_1 \dots \tau_k]_i$, whose root has colour $i\in\{1,\dots,d\}$ and $\tau_1,\dots,\tau_k$ are its immediate subtrees, set
\begin{equation}
    \label{eq:elementary_differentials}
    F_{\mathcal{W}}([\tau_1, \dots, \tau_k]_i) = \nabla^k_{F_\mathcal{W}(\tau_1),\dots,F_\mathcal{W}(\tau_k)} W^{(i)}.
\end{equation}

Intuitively, we apply the $k$-th covariant derivative to the base vector field $W^{(i)}$, and ask how much it moves along the directions of vector fields specified by the subtrees $\tau_1, \dots, \tau_k$. Crucially, if $\mathcal{H} = \hopfmkw$, the order of arguments in the covariant derivative $\nabla^k$ corresponds to the planar order of subtrees, and the chain rule for covariant derivatives. For non-planar trees, this construction uses a flat, torsion-free connection.
\section{Methodology}
\label{sec:methodology}

We define \gls{bnrde} by extending log-\gls{ncde} \cite{walker_log_2024} to non-Euclidean geometries. The model learns an algebra-valued vector field and integrates the dynamics using the log-ODE method. \Cref{fig:main_fig,alg:m_nrde_pseudocode} summarise the resulting piecewise log-ODE formulation.

\begin{algorithm}[h]
  \caption{\gls{bnrde}}
  \label{alg:m_nrde_pseudocode}
\begin{algorithmic}
\STATE {\bfseries Input:} Hopf algebra
\(\mathcal H\in\{\hopfshuf,\hopfgl,\hopfmkw\}\), truncation depth \(N\),
partition \(\{[t_k,t_{k+1}]\}_{k=0}^{M-1}\), path segments
\(\mathbf X_{t_k,t_{k+1}}\) with the chosen enhancement, vector-field lift \(F_W\), and initial condition
\(Y_0\in\mathcal M\)

    \STATE {\bfseries Precompute:} build the depth-\(N\), \(d\)-decorated basis of \(\mathcal H\), primitive basis
    \(\mathcal P\gets\mathcal B_{\mathcal H}^{\mathrm{prim}}\), signature
    routines, Hopf logarithm, and primitive-field evaluation schedules

    \STATE {\bfseries Cache offline:} for each enhanced segment \(\mathbf X_k\), compute
    \[
        \lambda_k
        =
        \log_{\mathcal H}\!\left(
            \operatorname{Sig}_{\mathcal H}^{N}(\mathbf X_k)
        \right)
        =
        \sum_{p\in\mathcal P}\lambda_k^p p .
    \]

    \FOR{\(k = 0\) {\bfseries to} \(M-1\)}
        \STATE
        \(L_k(Y)\gets
        \sum_{p\in\mathcal P}\lambda_k^p F_W(p)(Y)\)
        \hfill \(\triangleright\) lifted field

        \STATE
        \(\dot Z_\tau\!=\!L_k(Z_\tau),\, Z_0\!=\!Y_k;\,
        Y_{k+1}\gets Z_1\)
        \hfill \(\triangleright\) log-ODE step
    \ENDFOR

    \STATE {\bfseries Output:} trajectory samples \(\{Y_k\}_{k=0}^{M}\)
\end{algorithmic}
\end{algorithm}

\subsection{Signature Primitivisation}
\label{sec:signature_primitivisation}
Given a path \(X\) with a chosen enhancement, its \(\mathcal H\)-signature over \([s,t]\) is the group-like element \(\mathbb X^{\mathcal H}_{s,t}\in\widehat{\mathcal H}\) whose coordinates encode the corresponding iterated integrals, meaning that \(\Delta \mathbb X^{\mathcal H}_{s,t} = \mathbb X^{\mathcal H}_{s,t}\otimes \mathbb X^{\mathcal H}_{s,t}\). The log-ODE method requires expressing the signature in primitive elements via the Hopf logarithm:
\begin{equation}
\label{eq:hopf_log}
    \log_{\mathcal{H}}(g) \coloneqq \sum_{n \ge 1} \frac{(-1)^{n-1}}{n} (g - 1)^{\star n},
\end{equation}
for a group-like element $g$ where \(1\) is the unit and \(\star\) is the product in \(\mathcal{H}\). $\log_\mathcal{H}$ yields a primitive element, which we term the \(\log_\mathcal{H}\)-signature. At depth $N$, products and identities are retained through total degree $N$.

The product \(\star\) of the working Hopf algebra fixes the meaning of \eqref{eq:hopf_log}.
Throughout, \(\mathcal H\) denotes the graded dual in which enhanced paths are group-like and logarithms primitive in its completion, with coordinate identities such as the shuffle product rule stated dually. For \(\hopfshuf\), \(\star\) is tensor concatenation, whose dual coordinate product is the shuffle; for \(\mathcal H_{\mathrm{GL}}\) and \(\hopfmkw\), \(\star\) is the associative product extending tree grafting (\cref{app:gl_hopf,app:mkw_hopf}). A single Hopf logarithm thus yields the segment log-signatures \(\{\lambda_j\}_{j=0}^{T-1}\) used by B-NRDE across all three regimes.

\subsection{Neural Vector Fields}
\label{subsec:neural_lift}
The learnable component of a \gls{bnrde} is a collection of atomic driving vector fields, not the full primitive-indexed log-ODE field. For \(z\in\mathcal M\), we parameterise
\[
  W_\theta(z)
  =
  (W_\theta^{(1)}(z),\dots,W_\theta^{(d)}(z)),
  \quad
  W_\theta^{(i)}(z)\in T_z\mathcal M .
\]
Thus \(\mathcal W_\theta=(W_\theta^{(1)},\dots,W_\theta^{(d)})\) contains one vector field per driver channel. B-NRDE learns only these atomic channel fields. The primitive-indexed fields used by the log-ODE are then generated deterministically from \(W_\theta\) by the vector-field lifts described in the sequel.

The display above is \textit{representation-free}: the method only requires that the
network parameterisation determines an element of \(T_z\mathcal M\) for each
state \(z\) and driver channel. In our experiments, we use a homogeneous space realisation, in which the network outputs frame coordinates and the numerical flow is applied through the corresponding group action. This enforces the manifold constraint at every solver substep, avoiding the ambient-space projection or retraction errors that arise when integrating extrinsically.



\subsection{Vector Field Lifts}
\label{sec:vector_field_lifts}
We implement the pseudo bialgebra map as a \emph{vector field lift} $F_W\colon\mathcal B_{\mathcal H}^{\mathrm{prim}}\to\Gamma(TM)$ mapping primitive basis elements of the chosen Hopf algebra to vector fields on the state space. This follows from the coproduct Leibniz property of \cref{def:pseudo_bialgebra_map}, which ensures $F_\mathcal{W}(p)\in\Gamma(TM)$ for $p\in\mathcal B_{\mathcal H}^{\mathrm{prim}}$ \citep[Prop.~4.3]{kern_flow_2023}; see \Cref{appx:prop_primitive_vector_field}. In all cases, the implementation separates basis and combinatorial precomputation from \acrlong{ad}-based evaluation. The atomic fields are those defined in \Cref{subsec:neural_lift}, and we write $W \coloneqq \mathcal{W}_\theta,\ F_W \coloneqq F_{\mathcal{W}_\theta}$ for notational simplicity.

\textbf{Primitive bases and combinatorics.}
\emph{Words} ($\hopfshuf$): We use the Lyndon basis of the free Lie algebra, indexed by Lyndon words $\ell=(i_1,\dots,i_m)$ over $\{1,\dots,d\}$ with $|\ell|\leq N$; see \cref{def:lyndon_word}. We enumerate these words by degree using Duval's generator \citep{duval_generation_1988}. For each non-letter Lyndon word, we store its standard factorisation
\begin{equation}
\label{eq:standard_lyndon_factorisation}
\ell = uv, \qquad v \text{ the longest Lyndon suffix},
\end{equation}
as the pair of indices corresponding to $u$ and $v$.

\emph{Trees} ($\hopfgl,\hopfmkw$): GL primitives are spanned by decorated rooted trees, while MKW-dual primitives form the free Lie algebra on planar trees under the concatenation commutator. Trees are stored recursively as
\(
(\tau_1,\ldots,\tau_k,c),
\)
where $c\in\{1,\dots,d\}$ is the root colour and $\tau_1,\ldots,\tau_k$ are the root subtrees. We precompute the indices of these subtrees, giving the recursive evaluation schedule shared by both tree algebras. For $\hopfgl$, unordered rooted-tree shapes are enumerated in \citet{beyer_constant_1980} lexicographic order, with colourings identified under tree symmetries. For $\hopfmkw$, ordered rooted-tree shapes are generated recursively from ordered compositions of the remaining node count; the left-to-right order of children is part of the basis element, and all colourings are distinct.

\textbf{Shared differential-operator scheme.}
The lifted fields are defined intrinsically from vector fields on
\(\mathcal M\). For \(\hopfshuf\), Lyndon primitives are evaluated using the
ordinary Lie bracket of vector fields, while for \(\hopfgl\) and \(\hopfmkw\), tree primitives are evaluated using
iterated covariant derivatives with respect to the chosen connection
\(\nabla\). In the numerical implementation, these intrinsic operations are
evaluated by forward-mode automatic differentiation in the chosen coordinate
or frame representation.

\textbf{Evaluation: shuffle primitives (bracket recursion).}
Given atomic channel fields \(W^{(1)},\dots,W^{(d)}\), we set
\[
F_W(\ell)
=
\begin{cases}
W^{(i)}, & \ell=i,\\
[F_W(u),F_W(v)], & \ell=[u,v],
\end{cases}
\]
and the recursion terminates when the word is a single letter.

\textbf{Evaluation: tree primitives (multi-ary node recursion).} This recursion evaluates the tree part of the primitive space; planar forest combinations are evaluated by linear extension of \cref{eq:forest_lift}.
Fix a coloured rooted tree $p=(\tau,c)$. For each node $v$ of $\tau$, let $C(v)=(u_1,\dots,u_k)$ denote its ordered children (intrinsic for $\hopfmkw$; any fixed order for $\hopfgl$). We compute subtree fields $V_v$ in postorder. Leaves satisfy \(V_v=W^{(c(v))}\). If $v$ has ordered children $C(v)=(u_1,\dots,u_k)$, we set
\[
V_v
=
\nabla^k_{V_{u_1},\dots,V_{u_k}} W^{c(v)}.
\]
Here $\nabla^k$ is the iterated covariant derivative, evaluated recursively for arbitrary vector fields \(U_1,\dots,U_k,W\in\Gamma(T\mathcal M)\):
\begin{align*}
\nabla^k_{U_1,\dots,U_k}W
=
\nabla_{U_1}&\left(\nabla^{k-1}_{U_2,\dots,U_k}W\right)\\
&-
\sum_{j=2}^{k}
\nabla^{k-1}_{U_2,\dots,\nabla_{U_1}U_j,\dots,U_k}W .
\end{align*}
Finally,
\(
F_W(p)=V_r,
\)
for the root node $r$. This is the computational form of \eqref{eq:elementary_differentials}: each internal node is evaluated by JVP-based corrected covariant derivatives, and planarity in $\hopfmkw$ enters through the child order $C(v)$.

\subsection{The Manifold Log-ODE Method}
\label{sec:prelim_logode}

Let \(\mathcal H\) be a connected, cocommutative graded Hopf algebra, as
detailed in \Cref{app:graded_connected_truncated}. By
\cref{thm:milnor_moore},
\(\mathcal H\cong\mathcal U(\mathrm{Prim}(\mathcal H))\), so logarithmic
signature increments can be expressed in primitive coordinates. For a path
segment \(X_k\), write its \(\log_{\mathcal H}\)-signature as
\(\lambda_k=\sum_{p\in\mathcal B_{\mathcal H}^{\mathrm{prim}}}
\lambda_k^p p\in\mathrm{Prim}(\mathcal H)\). Given the vector-field lift
\(F_W\), the corresponding log-ODE field is
\begin{equation}
\label{eq:window_logode_field}
    L_k(Y)
    =
    \sum_{p\in\mathcal B_{\mathcal H}^{\mathrm{prim}}}
    \lambda_k^p F_W(p)(Y).
\end{equation}
The lifted field \(F_W(p)(Y)\) is evaluated by the recursions in
\Cref{sec:vector_field_lifts}: shuffle primitives use the Lie bracket
recursion, while tree primitives use the multi-ary covariant derivative
recursion. The vector-field lift is evaluated at the solver stages
whenever \(L_k\) is queried.

In Euclidean space, the segment update is obtained by solving the
normalised ODE
\(\dot Z_\tau=L_k(Z_\tau)\), \(Z_0=Y_k\), \(\tau\in[0,1]\), where the
segment length is already encoded in \(\lambda_k\). We solve this using Heun's second-order method and recover log-NCDE when, additionally, $\mathcal{H} = \hopfshuf$.

While our method works for any connection-equipped manifold, we make the implementation choice of using homogeneous spaces. As such, we represent tangent vectors in frame or Lie-algebra coordinates. Let $G$ be the Lie group acting on $\mathcal M$, and write $g\cdot Y$ for the action. For $\xi\in\mathfrak g$, denote the induced fundamental vector field by
\[
    \xi^\#(Y)
    =
    \left.\frac{\mathrm d}{\mathrm d\epsilon}\right|_{\epsilon=0}
    \exp(\epsilon \xi)\cdot Y .
\]
The lifted primitive evaluator returns frame coordinates
$\widehat F_W(p)(Y)\in\mathfrak g$ satisfying
\[
    F_W(p)(Y)
    =
    \widehat F_W(p)(Y)^\#(Y).
\]
Consequently, the window field is represented by
\begin{equation*}
    \widehat L_k(Y)
    =
    \sum_{p\in\mathcal B_{\mathcal H}^{\mathrm{prim}}}
    \lambda_k^p \widehat F_W(p)(Y),
    \quad
    L_k(Y)=\widehat L_k(Y)^\#(Y).
\end{equation*}
The log-ODE on the time interval $\tau\in[0,1]$ is then
\[
    \dot Z_\tau
    =
    \widehat L_k(Z_\tau)^\#(Z_\tau),
    \qquad
    Z_0=Y_k .
\]
This is a geometric numerical integration problem that may be solved using \acrlong{rkmk} or \gls{cf} methods. We select $\mathrm{CF\text{-}EES}(2,5)$ for its minimal exponential design which minimises per-step memory and compute costs \citep{shmelev2026expliciteffectivelysymmetricschemes}. We refer the reader to their text and appendices for an overview of commutator-free methods in neural differential equations. 

In the flat case, the group action reduces to addition, and the exponential update reduces to the corresponding explicit ODE method applied to \eqref{eq:window_logode_field}. Hence, the same implementation covers the Euclidean $\hopfgl$ case as well as the homogeneous-space $\hopfshuf$ and $\hopfmkw$ cases.
\subsection{Complexity Analysis}
At truncation depth \(N\), the algebraic cost is governed by the primitive basis size of the chosen Hopf algebra. Let \(b_{\mathcal H}(n)\) denote the number of degree-\(n\) primitive basis elements and let \(d\) be the path dimension. For \(\hopfshuf\), Witt's formula \cite{witt_treue_1937} gives
\[
b_{\shuffle}(n)=\frac1n\sum_{k\mid n}\mu(k)d^{n/k}.
\]
For \(\hopfgl\), primitives are \(d\)-decorated non-planar rooted trees counted up to colour-preserving isomorphism \cite{gobel_1-1-correspondence_1980}. For \(\hopfmkw\), the \(C_{n-1}d^n\) planar trees \cite{walkup_number_1972} generate a free Lie algebra under concatenation. The resulting primitive counts are derived in \cref{app:primitive_counts}.

Assume all models use the same atomic vector field \(f_\theta\), an \(m\)-layer \gls{mlp} with hidden width \(n_h\) and hidden-state dimension \(u\). For \(f_\theta\colon\mathbb R^u\to\mathbb R^{u\times d}\), one primal evaluation costs
\[
C_\theta=2u n_h+2(m-1)n_h^2+2u d n_h.
\]
At depth two, the \(d\) directional derivatives of the joint output give all \(DW_j[W_i]\), and hence all tree fields and Lie brackets. Writing \(C_{\mathrm{JVP}}\) for the cost of each such derivative gives
\[
C_{\mathrm{field}}=C_\theta+dC_{\mathrm{JVP}},
\qquad C_{\mathrm{JVP}}\approx3C_\theta,
\]
with additional contraction, connection and frame costs. The ordered-forest terms in the MKW lift are evaluated as in \cref{app:geometry_conventions}.

With \(J\) internal steps, \(R\) field evaluations and \(Q\) group actions per step, the per-window cost is approximately
\[
C_{\mathrm{window}}=JRC_{\mathrm{field}}+JQC_{\mathcal M}.
\]
Heun uses \((R,Q)=(2,0)\), while the principal update of \(\mathrm{CF\text{-}EES}(2,5)\) uses \((R,Q)=(3,3)\). Its embedded error estimate uses one further group action. The numerical group actions are specified in \cref{app:geometry_conventions}.

\begin{table}[h]
\caption{Neural and principal group-action per-window costs for \(N=2\). Here \(P_2\) is the primitive
dimension, \(A_2\) is the neural derivative multiplier
so that its derivative cost is \(3A_2C_\theta\), and \((R,Q)\) denotes
field-evaluation stages and group actions per principal solver step. \(\mathbb{R}, \hopfshuf\) is log-\gls{ncde}.}
\label{tab:complexity_instantiated}
\begin{center}
\begin{small}
\renewcommand{\arraystretch}{1.15}
\resizebox{\linewidth}{!}{
\begin{tabular}{@{}lcccc@{}}
\toprule
Space/Algebra & \(P_2\) & \(A_2\) & \((R,Q)\) & Cost/window \\
\midrule
\(\mathbb{R}, \hopfshuf\)
&
\(\frac{d(d+1)}{2}\)
&
\(d\)
&
\((2,0)\)
&
\(2J(1+3d)C_\theta\)
\\
\(\mathbb{R}, \hopfgl\)
&
\(d+d^2\)
&
\(d\)
&
\((2,0)\)
&
\(2J(1+3d)C_\theta\)
\\

\(\mathcal{M}, \hopfshuf\)
&
\(\frac{d(d+1)}{2}\)
&
\(d\)
&
\((3,3)\)
&
\(3J(1+3d)C_\theta+3JC_{\mathcal M}\)
\\

\(\mathcal{M}, \hopfmkw\)
&
\(d+d^2+\binom d2\)
&
\(d\)
&
\((3,3)\)
&
\(3J(1+3d)C_\theta+3JC_{\mathcal M}\)
\\
\bottomrule
\end{tabular}}
\end{small}
\end{center}
\end{table}
\vspace{-1em}
\subsection{The Branched Signature Kernel}
\label{sec:branched_signature_kernel}

The geometric signature kernel between two paths $x,y$ is
\[
    k_{\mathrm{geo}}(x,y)
    =
    \langle \mathrm{Sig}^{N}(x),\mathrm{Sig}^{N}(y)\rangle,
\]
where the inner product is taken on the depth-\(N\) truncated tensor signature space. Signature kernels provide a principled similarity measure for stochastic processes \citep{chevyrev_signature_2022} and have been used to train neural SDEs through signature-kernel scoring objectives \citep{issa_non-adversarial_2023}. In practice, training minimises
\begin{equation}
    \label{eq:geo_sig_ker}
    \begin{aligned}
    \mathcal{L}_{\mathrm{geo}}(\theta)
    =
    \mathbb{E}&_{X,X'\sim P_\theta}
    \big[k_{\mathrm{geo}}(X,X')\big]\\
    &-
    2\mathbb{E}_{X\sim P_\theta,Y\sim P_{\mathrm{data}}}
    \big[k_{\mathrm{geo}}(X,Y)\big],
    \end{aligned}
\end{equation}
where $P_{\mathrm{data}}$ denotes the data law and $P_\theta$ the model law. The term
\(
    \mathbb E_{Y,Y'\sim P_{\mathrm{data}}}
    \big[k_{\mathrm{geo}}(Y,Y')\big]
\)
is constant in \(\theta\).

Existing signature-kernel implementations \citep{salvi_signature_2021,cass_numerical_2025,toth_scalable_2025,toth_users_2025} compute geometric signatures, whose semimartingale limits correspond to iterated Stratonovich integrals. For semimartingale data observed on a grid, piecewise-linear interpolation has finite variation, so its canonical lift has zero bracket and does not expose quadratic covariation as a driver coordinate. Consequently, an It\^o or branched-signature model using this representation must supply bracket information separately.

\begin{definition}[Branched signature kernel objective]
Fix a Hopf algebra \(\mathcal H\), truncation depth \(N\), and the coordinate
basis used by \(\operatorname{Sig}_{\mathcal H}^N\); throughout,
\(\langle\cdot,\cdot\rangle_{\mathcal H_{\le N}}\) denotes the Euclidean
inner product on these coordinates. The kernel uses raw character coordinates, with the dual pairing given in \cref{app:duality_dictionary}. For drivers \(\mathbf X,\mathbf Y\) enhanced with their quadratic variation, define
\[
    k_{\mathrm{br}}^{N}(\mathbf X,\mathbf Y)
    =
    \left\langle
    \operatorname{Sig}_{\mathcal H}^{N}(\mathbf X),
    \operatorname{Sig}_{\mathcal H}^{N}(\mathbf Y)
    \right\rangle_{\mathcal H_{\le N}} .
\]
The branched signature-kernel objective is
\begin{equation}
    \label{eq:branched_sig_ker}
    \begin{aligned}
    \mathcal{L}_{\mathrm{br}}(\theta)
    =
    \mathbb{E}&_{\mathbf X,\mathbf X'\sim P_\theta}
    \big[k_{\mathrm{br}}^{N}(\mathbf X,\mathbf X')\big]\\
    &-
    2\mathbb{E}_{\mathbf X\sim P_\theta,\mathbf Y\sim P_{\mathrm{data}}}
    \big[k_{\mathrm{br}}^{N}(\mathbf X,\mathbf Y)\big].
    \end{aligned}
\end{equation}
As in \eqref{eq:geo_sig_ker}, the omitted data--data term is constant in \(\theta\).
\end{definition}

When bracket increments are available from the simulator, the objective uses them directly rather than reconstructing them from realised quadratic variation.
Appendix~\ref{app:finite_grid_bracket_kernel} formalises this finite-grid distinction. A geometric kernel computed from path values is a kernel on the projected path law, whereas the branched kernel is computed on the enhanced law containing bracket coordinates. A hypothetical bracket-aware baseline built only from grid samples must first reconstruct those coordinates, and is therefore subject to realised-covariance error. For standard continuous semimartingales, fixed-time realised-covariance fluctuations are typically of order \(\sqrt{\Delta_n}\), so supplying analytic or simulator-ground-truth brackets can remove this finite-grid error source without changing the continuous-limit law-matching target.

\section{Experiments}
We design our experimental evaluation to focus on regimes where standard Euclidean \glspl{nrde} struggle. We validate \gls{bnrde} across three distinct domains, mapping each to the Hopf algebra that captures its geometry and causality.

We first consider Euclidean rough volatility. Utilising $\hopfgl$, we task \gls{bnrde} with learning an unconditional generative model of the rough Bergomi stochastic volatility model, evaluating performance via the fidelity of the marginal laws. Next, we examine Stratonovich manifold dynamics under $\hopfshuf$ and evaluate forecasting performance. Finally, we address manifold-valued It\^o rough paths, employing $\hopfmkw$ to learn a generative model of covariance dynamics on the \gls{spdman} matrix manifold. Dataset code is available in \href{https://github.com/luke-a-thompson/RoughBench}{RoughBench}, and further experimental details may be found in \cref{app:experimental_details}.

\subsection{Baseline Models}

As baselines, we employ \gls{mnode} \cite{lou_neural_2020}, \gls{ncde} with linear \cite{kidger_neural_2020}, Hermite \cite{morrill_choice_2022}, and \gls{sg} \cite{bastian_continuous-time_2025} interpolation, dubbed \gls{ncde}, \gls{ncde}++ and \gls{sg}-\gls{ncde}, respectively. We also consider signature methods \gls{nrde} \cite{morrill_neural_2021} and log-\gls{ncde} \cite{walker_log_2024}, and discrete-time GRU \cite{chung2014empiricalevaluationgatedrecurrent}, xLSTM \cite{beckXLSTMExtendedLong2024}, and stacked xLSTM, the backbone of \citet{beckXLSTMExtendedLong2024}.

\subsection{Generative Modeling of Rough Volatility Under It\^o Integration}
\label{sec:rough_volatility}

Rough volatility models posit that log-volatility evolves with fractional regularity ($H < \tfrac{1}{2}$) \cite{gatheral_volatility_2014}. A prominent example is the \gls{rbergomi} model \cite{bayer_pricing_2015}, where volatility is driven by a Volterra process $W_t^H = \int_0^t K_H(t-s) \mathrm{d}W_s$.
Simulating these paths is a known bottleneck, as exact methods scale quadratically with the number of time steps. Furthermore, standard signature-based approaches are constrained to the geometric (Stratonovich) setting; capturing the financially-relevant It\^o integral requires augmenting the path with a lead-lag (Hoff) lift \cite{flint_discretely_2016}, doubling channel dimension from $d \to 2d$, massively growing signature size.

We use \gls{rbergomi} to stress-test the It\^o capability of the \gls{bnrde} framework in flat geometry ($\mathcal{M}=\mathbb{R}^d$). Unlike \gls{nrde}, our $\mathcal{H}_{\mathrm{GL}}$ formulation accommodates branched rough paths, allowing \gls{bnrde} to process the driving signal without lead--lag augmentation. Equation and simulation details can be found in \cref{appx:exp_det_rough_volatility}.
For training, we generate a time-augmented latent piecewise-linear Brownian motion $(t,Z_{\mathrm{lat}})$ on $[0,1]$ with grid size $\Delta_n$, and sample generated log-price paths $X \coloneqq \widehat{\log S}_\theta \sim P_\theta$. Ground-truth paths are $Y \coloneqq \log S \sim P_{\mathrm{data}}$. All models minimise $\mathcal{L}_{\mathrm{geo}}$ except \gls{bnrde}, which receives three additional fine-tuning epochs under the branched signature-kernel score between $(t,X)$ and $(t,Y)$, adding quadratic-variation/covariation terms $\langle X^i,X^j\rangle$ and $\langle Y^i,Y^j\rangle$; the former are the squared increments, the latter from the simulator ground truth.

In \cref{tab:rough_vol_results} we report \gls{ks} distances between model-generated and ground truth time-marginal laws. \gls{bnrde} achieves the strongest fit across three of four horizons, outperforming $\hopfshuf$-based \gls{nrde} and log-\gls{ncde}.

\begin{table}[t]
\centering
\caption{\gls{rbergomi} \gls{ks} across time-marginals}
\label{tab:rough_vol_results}
\resizebox{\linewidth}{!}{
\begin{tblr}{
  colspec = {l c *{4}{c}},
  cell{1}{1} = {r=2}{c},
  cell{1}{2} = {r=2}{c},
  cell{1}{3} = {c=4}{c},
  hline{1,Z} = {-}{0.08em},
  hline{2} = {3-6}{},
  hline{3} = {1-6}{},
}
\textbf{Method}
& \textbf{\shortstack{Training\\Time (s)}}
& \textbf{KS Score ($\times 10^{-2}$)} &  &  &  \\
& & \textbf{128} & \textbf{256} & \textbf{384} & \textbf{512} \\
xLSTM & 38 & $11.04 \scriptstyle{\pm 1.38}$ & $10.31 \scriptstyle{\pm 1.30}$ & $10.94 \scriptstyle{\pm 0.89}$ & $11.71 \scriptstyle{\pm 0.71}$ \\
\gls{ncde} & 1154 & $9.06 \scriptstyle{\pm 1.94}$ & $8.93 \scriptstyle{\pm 1.07}$ & $9.75 \scriptstyle{\pm 1.02}$ & $9.65 \scriptstyle{\pm 0.81}$ \\
\gls{ncde}++ & 12021 & $8.11 \scriptstyle{\pm 0.93}$ & $8.82 \scriptstyle{\pm 0.86}$ & $9.67 \scriptstyle{\pm 0.34}$ & $10.73 \scriptstyle{\pm 0.52}$ \\
Stacked xLSTM & 313 & $\underline{7.69 \scriptstyle{\pm 0.16}}$ & $9.19 \scriptstyle{\pm 1.09}$ & $9.50 \scriptstyle{\pm 1.21}$ & $10.95 \scriptstyle{\pm 1.11}$ \\
log-\gls{ncde} & 74 & $7.75 \scriptstyle{\pm 1.39}$ & $\underline{7.38 \scriptstyle{\pm 0.46}}$ & $8.72 \scriptstyle{\pm 0.88}$ & $8.89 \scriptstyle{\pm 0.24}$ \\
\gls{nrde} & 157 & $8.74 \scriptstyle{\pm 0.96}$ & $8.41 \scriptstyle{\pm 0.94}$ & $8.12 \scriptstyle{\pm 0.79}$ & $\mathbf{7.47 \scriptstyle{\pm 0.83}}$ \\
GRU & 42 & $7.73 \scriptstyle{\pm 1.78}$ & $7.59 \scriptstyle{\pm 0.58}$ & $\underline{7.87 \scriptstyle{\pm 0.67}}$ & $\underline{8.02 \scriptstyle{\pm 1.15}}$ \\
\hline
\textbf{\gls{bnrde}} (GK) & 233 & $\mathbf{6.89 \scriptstyle{\pm 1.81}}$ & $\mathbf{6.91 \scriptstyle{\pm 0.98}}$ & $8.56 \scriptstyle{\pm 0.38}$ & $9.58 \scriptstyle{\pm 0.36}$ \\
\textbf{\gls{bnrde} (BK)} & 47 & $7.93 \scriptstyle{\pm 1.49}$ & $7.70 \scriptstyle{\pm 0.87}$ & $\mathbf{7.77 \scriptstyle{\pm 1.02}}$ & $8.11 \scriptstyle{\pm 0.52}$ \\
\end{tblr}

}
\end{table}
\begin{figure}[h]
    \centering
    \includegraphics[width=\linewidth]{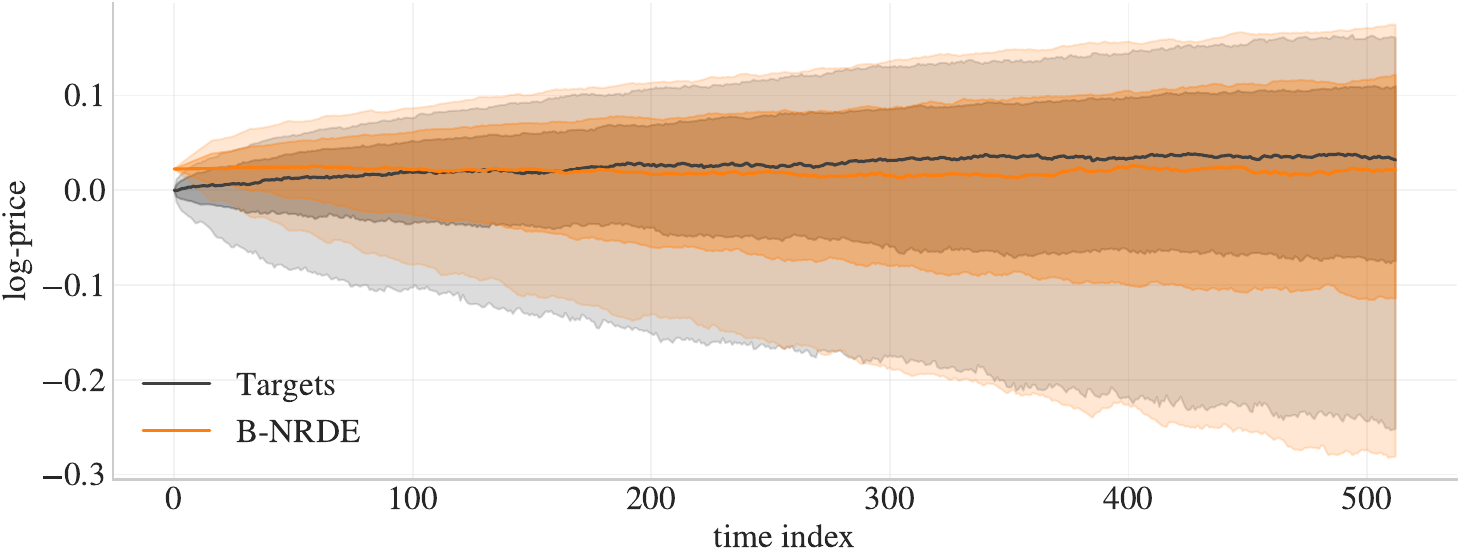}
    \caption{\gls{rbergomi} sample paths generated by \gls{bnrde} showing good fit to the underlying law.}
    \label{fig:rbergomi_bnrde_fan}
    \vspace{-1em}
\end{figure}

\label{sec:motion_dynamics}
\subsection{Sim-to-Real Dynamics Forecasting}
Three-dimensional motion forecasting is central to robotics and computer vision, with applications to pose prediction \citep{yang_dynamical_2021, dunkel_normalizing_2024} and occlusion-robust perception \cite{di_so-pose_2021}. We consider sim-to-real forecasting of rotational dynamics, previously out of reach for signature methods due to the lack of a manifold-compatible formulation. As a deterministic task without an It\^o dynamic, we leverage a coordinate-free $\hopfshuf$ formulation.
\begin{figure}[t]
    \centering
    \includegraphics[width=0.98\linewidth]{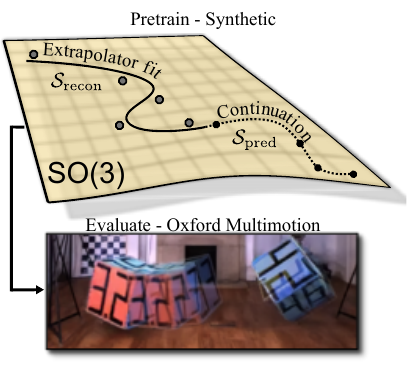}
    \caption{Schematic description of the sim-to-real dynamics training pipeline. Adapted from \citet{bastian_continuous-time_2025}.}
    \label{fig:sim_to_real}
    \vspace{-1em}
\end{figure}

Following \citet{bastian_continuous-time_2025}, we generate offline continuous-time rotation trajectories \(R\colon[0,1]\to \mathrm{SO}(3)\), where \(R(t)\) is the orientation at time \(t\).
From each trajectory we extract sliding windows \(\{(t_j, R(t_j))\}_{j=0}^{m}\) and split each window into a reconstruction prefix \(\mathcal{S}_{\mathrm{recon}}\) and prediction suffix \(\mathcal{S}_{\mathrm{pred}}\). An extrapolator \(\mathsf{E}\) is fit on \(\mathcal{S}_{\mathrm{recon}}\) and evaluated on the full window to produce a dense input path \(\tilde R(t)\) over \(\mathcal{S}_{\mathrm{recon}}\cup\mathcal{S}_{\mathrm{pred}}\). This extrapolated path is then provided to the model, which outputs predicted rotations \(\{\hat R_{\theta}(t_j)\}_{j=0}^{m}\) over the same times. We pretrain by minimising the Frobenius discrepancy to the ground truth over the full window:
\[
\min_{\theta}\; \sum_{j\in\mathcal{J}} \bigl\| \hat R_{\theta}(t_j) - R(t_j) \bigr\|_F^2,
\quad \hat R_{\theta}(t_j),\, R(t_j)\in \mathrm{SO}(3),
\]
where \(\mathcal{J}=\{0,\dots,m\}\).

To assess sim-to-real generalisation, simulation-pretrained models are evaluated on the \gls{omd} \citep{judd_oxford_2019}, with windows and extrapolations identical to those used for the simulation data. The rotation geodesic error (RGE) \citep{huynh_metrics_2009}, defined as \(\operatorname{RGE}(R_1, R_2) = 2 \arcsin \left( \frac{\|R_2-R_1\|_F}{2 \sqrt{2}} \right)\), is computed over both segments. \Cref{fig:sim_to_real} illustrates the protocol.

\begin{table}[t]
\centering
\caption{Training test set Frobenius norm and RGE (degrees) across different motion scenarios in the \gls{omd} dataset.}
\label{table:omd_dataset}
\resizebox{\linewidth}{!}{%
\begin{tblr}{
  colspec = {l *{4}{c}},
  hline{1-2,9-10} = {1,2-5}{},
}
Method                       & {Pretraining\\Frobenius} & {Static\\ Motion}        & {Translation\\ Motion}   & {Unconstrained\\ Motion} \\
\gls{mnode}          & 2.461                    & 132.40                   & 118.25                   & 121.08                   \\
$\mathrm{SO}(3)$-\gls{ncde}  & 0.294                    & 22.21                    & 21.83                    & 22.45                    \\
$\mathrm{SO}(3)$-GRU         & 0.271                    & 20.25                    & 20.39                    & 21.13                    \\
xLSTM                   & 0.233                    & 16.67                     & 16.77                     & 17.36                     \\
\gls{nrde}                   & 0.097                    & 7.45                     & 7.34                     & 7.32                     \\
Stacked xLSTM                   & 0.110                    & 7.11                     & 7.50                     & 7.03                     \\
SG-\gls{ncde}                & 0.050                    & \textbf{2.93}            & \textbf{3.31}            & \textbf{2.80}            \\
\textbf{\gls{bnrde} (GK)}        & \textbf{0.049}           & \underline{3.23}         & \underline{3.70}         & \underline{3.33}         \\
\end{tblr}
}
\end{table}

Whilst SG-\gls{ncde} attains the lowest RGE, \gls{bnrde} achieves comparable accuracy with only two solver steps versus 20. The remaining small deficit is not explained by the simulation pretraining loss, suggesting modest degradation in sim-to-real transfer. Crucially, \gls{bnrde} generalises the SG-\gls{ncde} pipeline by admitting rough drivers, thereby removing the $C^1$ interpolant requirement and permitting non-smooth extrapolators such as \glspl{mlp}.

\subsection{It\^o Dynamics over the SPD Manifold}

\gls{spdman}-covariance dynamics learning arises in multivariate finance \citep{noureldin_multivariate_2012, johansson_simple_2023}, medical imaging \citep{pennec_manifold-valued_2020}, and, more recently, Riemannian diffusion models \citep{park_riemannian_2022, de_bortoli_riemannian_2022, thornton_riemannian_2022}. Notably, finance canonically requires non-anticipative It\^{o} modelling to avoid lookahead bias, and diffusion models are typically formulated as It\^{o} diffusions.

As such, we study a mean-reverting It\^{o} diffusion on the SPD manifold $\mathbb{S}^{++}_d$ endowed 
with the \gls{aiman} geometry, testing \gls{bnrde}'s ability to learn non-anticipative 
stochastic dynamics on a curved state space while strictly preserving the \gls{spdman} 
cone. Let \(M \in \mathbb{S}^{++}_d\) and let \(X_0 \in \mathbb{S}^{++}_d\). Consider the OU-type Itô diffusion
\begin{equation}\label{eq:affine_invariant_diffusion}
    \mathrm{d}^{\nabla}X_t
    =
    \eta\,\log_{X_t}(M)\,\mathrm{d}t
    +
    \sigma\,X_t^{1/2}\,\mathrm{d}B_t\,X_t^{1/2},
\end{equation}
where $\mathrm d^\nabla$ denotes It\^o integration for the affine-invariant Levi-Civita connection, \(\eta>0\) controls mean reversion, \(\sigma\ge 0\) is the noise scale, and \(B_t\in\mathrm{Sym}(d)\) is a symmetric matrix Brownian motion. In practice, the vectorised state $\tilde{x}_t = \mathrm{vech}(X_t)\in\mathbb{R}^{q}$, $\mathrm{d}\langle \tilde{x} \rangle_t \in \mathbb{R}^{q \times q}$ with $q = \frac{d(d+1)}{2}$ is used throughout.

Training minimises an expected signature-kernel objective between the ground-truth and generated trajectories, with log-\gls{ncde} using the geometric signature kernel and \gls{bnrde}. Since $\mathbb S^{++}_d$ is not totally ordered, a \gls{ks} statistic is not directly applicable. We therefore sort the eigenvalues increasingly and report empirical time-marginal \(W_1\) distances between the eigenvalue vectors $\lambda(X_t), \lambda(\widehat{X}_t)\in\mathbb{R}^d$.

In \cref{tab:spd_manifold_results} we find that the manifold-preserving capabilities of \gls{bnrde} improve on the results of Euclidean log-\gls{ncde} in terms of law-matching, showing an average improvement of 56.2\%. However, visual inspection of \cref{fig:spd_eigenvalue_trajectories} shows somewhat disappointing fidelity near the initial condition, a common pitfall of signature-based methods that are invariant to the starting value without augmentation \citep{morrill_generalised_2021}. We also find no meaningful performance difference between the branched and geometric kernels in this setting.

\begin{table}[h]
\centering
\caption{1-Wasserstein distance between predicted and ground-truth eigenvalue marginals on $\mathbb{S}^{++}_d$.}
\label{tab:spd_manifold_results}
\resizebox{\linewidth}{!}{
\begin{tblr}{
    colspec = {l *{5}{c}},
    cell{1}{1} = {r=2}{c},
    cell{1}{2} = {r=2}{c},
    cell{1}{3} = {c=4}{c},
    hline{1,Z} = {-}{0.08em},
    hline{2} = {3-6}{},
    hline{3} = {1-6}{},
    hline{9} = {1-6}{},
  }
  \textbf{Method}
  & \textbf{\shortstack{Training\\Time (s)}}
  & \textbf{1-Wasserstein ($\times 10^{-2}$)} &  &  &  \\
  & & \textbf{128} & \textbf{256} & \textbf{384} & \textbf{512} \\
  \gls{ncde}++   & $1605$ & $64.95 \scriptstyle{\pm 0.98}$ & $58.27 \scriptstyle{\pm 1.36}$ & $50.56 \scriptstyle{\pm 1.73}$ & $42.15 \scriptstyle{\pm 2.22}$ \\
  \gls{ncde}     & $1241$ & $65.56 \scriptstyle{\pm 1.74}$ & $58.72 \scriptstyle{\pm 2.38}$ & $50.32 \scriptstyle{\pm 3.02}$ & $42.38 \scriptstyle{\pm 3.11}$ \\
  \gls{nrde}     & $232$ & $11.80 \scriptstyle{\pm 0.15}$ & $15.50 \scriptstyle{\pm 0.16}$ & $17.36 \scriptstyle{\pm 0.13}$ & $17.70 \scriptstyle{\pm 0.28}$ \\
  xLSTM          & $12$ & $8.36 \scriptstyle{\pm 0.63}$ & $10.49 \scriptstyle{\pm 0.41}$ & $13.16 \scriptstyle{\pm 0.39}$ & $14.09 \scriptstyle{\pm 0.53}$ \\
  log-\gls{ncde} & $481$ & $12.10 \scriptstyle{\pm 0.30}$ & $14.76 \scriptstyle{\pm 0.22}$ & $16.17 \scriptstyle{\pm 0.32}$ & $16.81 \scriptstyle{\pm 0.40}$ \\
  Stacked xLSTM  & $18$ & $\mathbf{5.50 \scriptstyle{\pm 0.10}}$ & $6.85 \scriptstyle{\pm 0.35}$ & $11.47 \scriptstyle{\pm 0.63}$ & $14.30 \scriptstyle{\pm 0.46}$ \\
  \textbf{\gls{bnrde} (GK)} & $495$ & $\underline{5.73 \scriptstyle{\pm 0.20}}$ & $\mathbf{5.81 \scriptstyle{\pm 0.19}}$ & $\mathbf{6.28 \scriptstyle{\pm 0.86}}$ & $\underline{8.35
  \scriptstyle{\pm 0.94}}$ \\
  \textbf{\gls{bnrde} (BK)} & $502$ & $5.81 \scriptstyle{\pm 0.20}$ & $\underline{5.87 \scriptstyle{\pm 0.21}}$ & $\underline{6.30 \scriptstyle{\pm 0.85}}$ & $\mathbf{8.32 \scriptstyle{\pm
  0.94}}$ \\
\end{tblr}
}
\end{table}

\begin{figure}[H]
    \centering
    \begin{subfigure}[t]{0.9\linewidth}
        \centering
        \includegraphics[width=\linewidth]{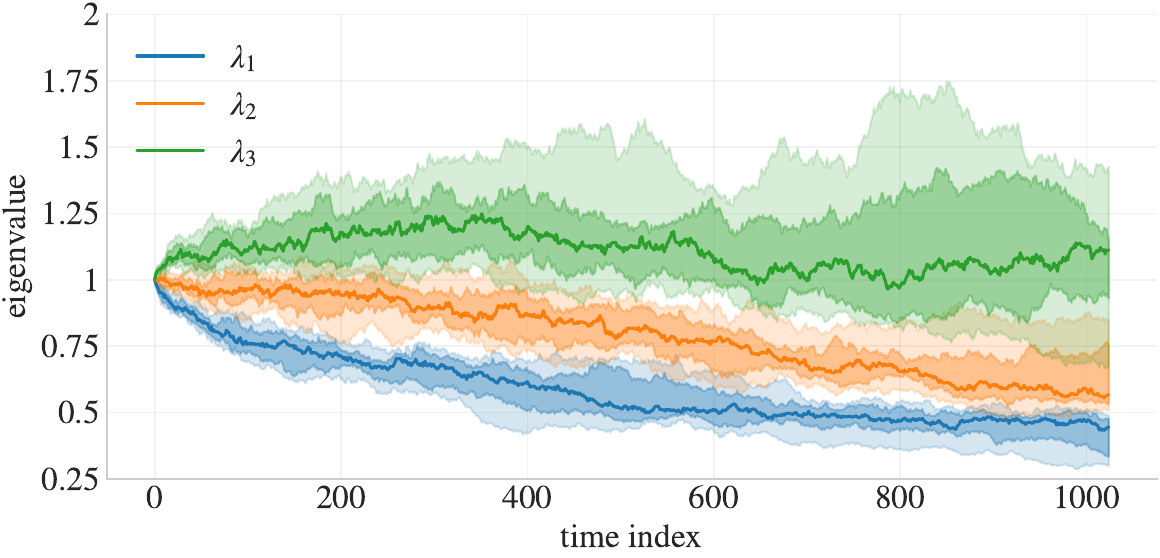}
        \caption{Ground truth eigenvalue trajectories of \cref{eq:affine_invariant_diffusion}.}
        \label{fig:spd_eigenvalue_targets}
    \end{subfigure}

    \vspace{0.75em}

    \begin{subfigure}[t]{0.9\linewidth}
        \centering
        \includegraphics[width=\linewidth]{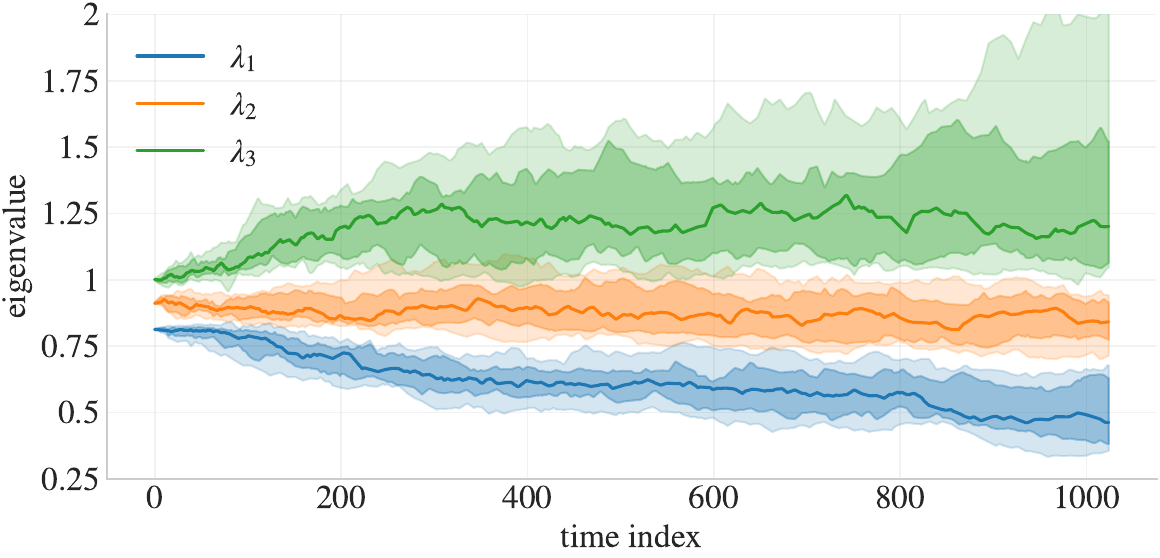}
        \caption{Predicted eigenvalue trajectories output by \gls{bnrde}.}
        \label{fig:spd_eigenvalue_preds}
    \end{subfigure}

    \caption{Ground truth vs.\ \gls{bnrde} eigenvalue trajectories.}
    \label{fig:spd_eigenvalue_trajectories}
\end{figure}

\section{Discussion}
In this work, we introduced \gls{bnrde}, a generalisation of \gls{nrde} to branched rough paths, and a branched-signature-kernel training objective for neural \glspl{sde}. Our method extends signature methods beyond the geometric/Stratonovich setting to It\^{o}, manifold, and It\^{o}-on-manifold dynamics, while remaining empirically competitive. We summarise key limitations of the current method and outline the most promising extensions.

\subsection{Limitations}
Unlike the untruncated PDE-based solvers available for geometric kernels \citep{salvi_signature_2021}, branched signature kernels require explicit truncation, discarding higher-order iterated-integral information at greater memory and compute cost. We expect this to worsen as the state dimension $d$ grows, since the quadratic-variation features scale as $d^2$.

\subsection{Future Directions}
\paragraph{Numerical} Such PDE-based solvers, or other approximation schemes, offer the most promising route to avoiding explicit truncation in our kernel. Adaptive log-ODE schemes could also permit larger signature windows and faster integration \cite{bayerAdaptiveAlgorithmRough2023}.

\paragraph{Signatures} The log-signature quotients out the shuffle identities, projecting onto the strictly smaller Lyndon-word basis. The branched logarithms likewise admit primitive bases of trees for GL and Lie combinations of planar trees for the MKW dual. Coordinate-side primitive projections \cite{bellingeri_branched_2024} offer further opportunities to improve the memory efficiency of B-NRDE. For scalar \glspl{rde}, such as the rough Bergomi model of \cref{sec:rough_volatility}, the theory of multi-index rough paths, elements of the Linares--Otto-Tempelmayr Hopf algebra \citep{Linares_2023}, offers an even more compact representation by exploiting symmetries in the vector fields of scalar equations \cite{Bellingeri_2026}.

\paragraph{SPDEs} The rough-path philosophy extends from finite-dimensional differential equations to \glspl{spde} via the theory of regularity structures \cite{hairer_theory_2014}. Much like how planarly branched rough paths enable solving \glspl{rde} over manifolds, planar regularity structures \cite{rahm_planar_2022} provide the corresponding generalisation for regularity structures and \glspl{spde}. \Cref{com:future_dir} recapitulates \citet[Fig.~1]{rahm_planar_2022} and points to a natural future direction: a neural planar-regularity-structure model in a $\mathbb{R}^d \to \mathcal{M}$ setting, beyond the $\mathbb{R}\to\mathcal{M}$ framework of \gls{bnrde} and the $\mathbb{R}^d\to\mathbb{R}^d$ framework of Neural Operator with Regularity Structure \cite{hu_neural_2022} and Deep Latent Regularity Net \cite{gong_deep_2023}.

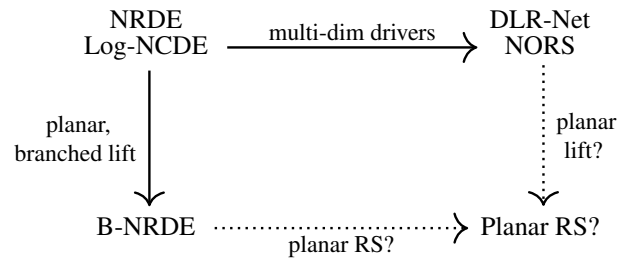
\begin{figure}[H]
\centering
\begin{tikzcd}[
  row sep=18mm,
  column sep=32mm,
  cells={nodes={align=center}},
  every arrow/.append style={line width=0.9pt},
  every label/.append style={font=\small},
]
  \shortstack{NRDE\\Log-NCDE}
    \arrow[r, "\text{multi-dim drivers}"]
    \arrow[d, "\shortstack{\text{planar,}\\\text{branched lift\ }}"']
&
  \shortstack{DLR-Net\\NORS}
    \arrow[d, dotted, "\shortstack{\text{\ planar}\\\text{lift?}}"]
\\
  \text{B-NRDE}
    \arrow[r, dotted, "\text{planar RS?}"']
&
  \text{Planar RS?}
\end{tikzcd}
\caption{Positioning of our method and a hypothetical planar-regularity-structure extension.}
\label{com:future_dir}
\end{figure}

\clearpage

\section*{Acknowledgements}
We thank Slade Matthews for his supervision, guidance, and support. We also thank Daniil Shmelev for providing his software library, pySigLib, which we used extensively in our work, and for rapidly implementing the requested features and bug fixes needed for the experiments.

\section*{Impact Statement}
This paper presents fundamental work in time-series forecasting aimed at advancing the field of Machine Learning. There are many potential societal consequences of our work, but none of which we feel must be specifically highlighted here.

\bibliography{shared/references}
\bibliographystyle{icml2026}

\newpage
\appendix
\onecolumn
\crefalias{section}{appendix}
\glsresetall

\section{The Hopf Algebras of Rough Paths}
\label{appx:hopf_algebras_of_rough_paths}

This appendix records the algebraic structures used throughout the paper. We first detail the formal requirements imposed on our Hopf algebras, then explain why we use
\begin{enumerate*}[label=(\roman*)]
    \item the \emph{graded dual} of the shuffle Hopf algebra for geometric, Stratonovich-type drivers,
    \item \gls{gl} rather than \gls{bck} for branched, It\^o-type drivers, and
    \item the \emph{graded dual} of the \gls{mkw} Hopf algebra for the manifold log-ODE scheme.
\end{enumerate*}
We keep the presentation concise and refer to \citet{kern_flow_2023} and \citet{manchon_rough_2025} for a more comprehensive account.

\paragraph{Duality and normalisation.}
\label{app:duality_dictionary}
For a connected graded coordinate Hopf algebra $A$ with finite-dimensional homogeneous components, write $H=A^\circ=\bigoplus_{n\ge0}A_n^*$ and $\widehat H=\prod_{n\ge0}H_n$. Characters of $A$ correspond to group-like series in $\widehat H$, and convolution corresponds to the product $\star$ in $H$. Their logarithms are infinitesimal characters of $A$ and primitive series in $\widehat H$.

The word and ordered-forest bases use the Kronecker pairing. For non-planar forests, the unnormalised grafting basis instead satisfies
\[
\langle f,h\rangle=\sigma(f)\delta_{f,h},\qquad
g_x=\sum_f\frac{x(f)}{\sigma(f)}f,
\]
where $\sigma(f)$ counts colour-preserving automorphisms, including permutations of identical components. For example, $\sigma(\tomato{i}\tomato{i})=\sigma(\cherry{j}{i}{i})=2$. Roughrax contracts raw BCK logarithmic coordinates with elementary differentials divided by $\sigma$. The signature kernel uses the raw character coordinates.

\subsection{Hopf Algebras}
\label{app:graded_connected_truncated}

\paragraph{Gradedness.}
A Hopf algebra \(\mathcal H\) is graded if it admits a decomposition
\[
    \mathcal H=\bigoplus_{n\ge 0}\mathcal H_n
\]
such that its product \(\star\) and coproduct \(\Delta\) respect degree:
\begin{equation}
    \mathcal{H}_m\star \mathcal{H}_n \subseteq \mathcal{H}_{m+n},
    \qquad
    \Delta(\mathcal{H}_n)\subseteq \bigoplus_{m=0}^n \mathcal{H}_m\otimes \mathcal{H}_{n-m}.
\end{equation}
The component \(\mathcal H_n\) corresponds to level \(n\) signature coordinates. For word signatures, \(n\) is word length; for branched signatures, \(n\) is the number of vertices.

\paragraph{Connectedness.}
A graded Hopf algebra is connected if the degree-zero component is spanned by the unit:
\[
    \mathcal{H}_0=\operatorname{span}\{1\}.
\]
For instance, in the \gls{gl} Hopf algebra on rooted forests, the empty forest is the unit \(1\), and every non-empty forest has positive degree.

\paragraph{Truncation.}
In applications, we truncate at depth \(N\):
\[
    \mathcal{H}^{(\le N)} := \bigoplus_{n=0}^{N}\mathcal{H}_n .
\]
Products and coproduct identities are retained through total degree $N$, with full signatures and logarithms interpreted in $\widehat H$. In particular, the tensor-product target is truncated by total degree, giving finite polynomial expressions for the numerical scheme.

\paragraph{Primitives and grouplike elements.}
Primitive elements have no nontrivial coproduct:
\[
    \operatorname{Prim}(\mathcal{H})
    :=
    \{p\in\mathcal H:\Delta p=p\otimes 1+1\otimes p\}.
\]
For rough paths, increments are characters of the coordinate Hopf algebra; equivalently, they are grouplike elements of the graded dual. The Hopf logarithm \(\log_{\star}\), as in \cref{eq:hopf_log}, maps grouplike elements to primitive elements. Indeed, $\Delta g=(g\otimes1)(1\otimes g)$ and these two factors commute, so $\Delta\log_\star g=\log_\star g\otimes1+1\otimes\log_\star g$.

\subsection{Why cocommutativity matters: Milnor--Moore}
\label{sec:milnor_moore}

The log-ODE construction is expressed at the level of primitives: one takes the Hopf logarithm of a grouplike increment, obtains a Lie element, and then exponentiates the resulting truncated Lie polynomial in differential operators. This gives a compression from the full Hopf algebra to its primitive part. To recover the Hopf algebra from its primitives, we use the Milnor--Moore theorem.

\begin{theorem}[Milnor--Moore \cite{milnorStructureHopfAlgebras1965}]
\label[theorem]{thm:milnor_moore}
Let \(\mathcal H\) be a connected, graded, cocommutative Hopf algebra over a field of characteristic \(0\). Then there is a canonical Hopf algebra isomorphism
\begin{equation}
    \mathcal H \cong U\bigl(\operatorname{Prim}(\mathcal H)\bigr).
\end{equation}
\end{theorem}

Consequently, we do not use the noncocommutative coordinate Hopf algebras \gls{bck} and \gls{mkw} directly for the primitive log-ODE construction. Instead, we use their cocommutative graded duals, namely \(\gls{gl}\) and \((\gls{mkw})^{\circ}\).

\subsection{Words and Stratonovich calculus}

\subsubsection{Words and Lyndon words}
\label{app:words}

Throughout this subsection, let \(A\) be a finite alphabet, typically \(A=\{1,\ldots,d\}\), with a fixed total order. A \emph{word} is a finite sequence \(w=i_1\cdots i_n\) with \(i_k\in A\), including the empty word \(1\) when \(n=0\). Its length is \(|w|=n\), concatenation is denoted by juxtaposition, and the order on \(A\) induces the lexicographic order on words.

\begin{definition}[Lyndon word]
\label[definition]{def:lyndon_word}
A nonempty word \(\ell\) is a \emph{Lyndon word} if it is strictly smaller, in lexicographic order, than each of its nontrivial proper suffixes.
\end{definition}

\begin{example}[Lyndon bracketing]
Let \(A=\{a,b,c,d\}\) with \(a<b<c<d\). The word
\(\ell=\texttt{abac}\) is Lyndon, since
\[
    \texttt{abac}<\texttt{bac},
    \qquad
    \texttt{abac}<\texttt{ac},
    \qquad
    \texttt{abac}<\texttt{c}.
\]
Its standard factorisation is \(\ell=uv\), with \(u=\texttt{ab}\) and \(v=\texttt{ac}\). Hence
\[
    [\ell]=[[u],[v]]
    =
    [[e_a,e_b],[e_a,e_c]].
\]
As \(\ell\) ranges over Lyndon words, these bracketings form the Lyndon basis of the free Lie algebra \(\operatorname{Prim}(\hopfshuf)\).
\end{example}

\subsubsection{Stratonovich rough paths on \(\mathbb{R}^n\): the tensor Hopf algebra}
\label{app:shuffle_hopf}

We view the geometric signature as a grouplike element of the cocommutative graded dual of the shuffle Hopf algebra,
\[
    \mathcal{H}_{\shuffle}^{\circ}
    \coloneqq
    \bigl(T(A),\,\star_{\shuffle},\,\Delta_{\mathrm{unsh}}\bigr),
\]
where \(\star_{\shuffle}\) is concatenation and \(\Delta_{\mathrm{unsh}}\) is the unshuffle coproduct. This is the graded dual of the usual shuffle Hopf algebra \((T(A),\shuffle,\Delta_{\mathrm{dec}})\). For words \(u=i_1\cdots i_m\) and \(v=j_1\cdots j_n\),
\[
    u\star_{\shuffle} v
    \coloneqq
    i_1\cdots i_m j_1\cdots j_n .
\]
The unshuffle coproduct is
\[
    \Delta_{\mathrm{unsh}}(i_1\cdots i_n)
    \coloneqq
    \sum_{I\subseteq [n]}
    i_I\otimes i_{I^{c}},
    \qquad
    i_I := i_{k_1}\cdots i_{k_{|I|}}
    \quad
    \text{for } I=\{k_1<\cdots<k_{|I|}\}.
\]
It is cocommutative, since swapping tensor factors corresponds to \(I\leftrightarrow I^{c}\). The primitives
\(\operatorname{Prim}(\mathcal H_{\shuffle}^{\circ})\) form the free Lie algebra on \(A\), with bracket induced by the commutator of \(\star_{\shuffle}\). Therefore the Hopf logarithm of any geometric signature increment is a Lie element, and may be expanded in the Lyndon basis.

\begin{example}[Smooth tensor signature]
Let \(X_t^a=t\) and \(X_t^b=t^2\) on \([0,1]\). Then
\[
    \langle S(X),a\rangle=\int_0^1 dX_t^a=1,
    \qquad
    \langle S(X),b\rangle=\int_0^1 dX_t^b=1,
\]
and
\[
    \langle S(X),ab\rangle
    =
    \int_{0<u<v<1} dX_u^a\,dX_v^b
    =
    \int_0^1 v\,2v\,dv
    =
    \frac{2}{3},
\]
while
\[
    \langle S(X),ba\rangle
    =
    \int_{0<u<v<1} dX_u^b\,dX_v^a
    =
    \int_0^1 v^2\,dv
    =
    \frac{1}{3}.
\]
Thus
\[
    \langle S(X),ab+ba\rangle
    =
    \langle S(X),a\rangle\langle S(X),b\rangle,
\]
as required by the shuffle character relation. With \(V_a=\partial_y\) and \(V_b=y\partial_y\), and with the convention that \(ab\) acts by \(V_aV_b\),
\[
    V_aV_b\varphi=\varphi'+y\varphi'',
    \qquad
    V_bV_a\varphi=y\varphi'',
\]
so the two tensor coordinates act on distinct differential operators even for this one-dimensional example.
\end{example}

\subsection{Rooted trees and It\^o calculus}

\subsubsection{Non-planar and planar rooted trees}
\label{app:rooted_trees}

\begin{definition}[Rooted tree]
\label[definition]{def:rooted_tree}
A rooted tree is a finite connected graph with no undirected cycles, a distinguished root, and optional vertex decorations from an alphabet $A$, with edges oriented consistently towards the root. The tree with root decorated by \(i\) and children \(\tau_1,\ldots,\tau_n\) is denoted
\begin{equation*}
    [\tau_1,\dots,\tau_n]_i
    =
    \begin{tree}
        \draw (0,0)\mynode{i} -- (-2.7ex,2.5ex)\mynode{$\tau_1$};
        \draw (0,0) -- (-.9ex,1.5ex);
        \draw (0,0) -- (0,1.5ex);
        \draw (0,0) -- (.9ex, 1.5ex);
        \draw (0.15,2.5ex) node[draw = none, fill = none]{$\dots$};
        \draw (0,0) -- (2.7ex,2.5ex)\mynode{$\tau_n$};
    \end{tree}.
\end{equation*}
The order \(|\tau|\) is the number of vertices, and a forest is a finite product of rooted trees.
\end{definition}

We distinguish \emph{non-planar} rooted trees, whose children are unordered, from \emph{planar} rooted trees, whose children are ordered from left to right. Thus, in the planar setting,
\[
    \cherry{i}{j}{k}\neq \cherry{i}{k}{j},
\]
whereas these two trees agree in the non-planar setting.

\subsubsection{Pre-Lie grafting and Euclidean elementary differentials}
\label{app:prelie_grafting}

The Euclidean branched theory is governed by a pre-Lie product. A pre-Lie algebra is a vector space with a bilinear product \(\triangleright\) satisfying
\[
    x\triangleright (y\triangleright z)
    -
    (x\triangleright y)\triangleright z
    =
    y\triangleright (x\triangleright z)
    -
    (y\triangleright x)\triangleright z .
\]
For vector fields on \(\mathbb R^n\), this product is
\[
    U\triangleright V
    \coloneqq
    DV[U].
\]
The free pre-Lie algebra is represented by non-planar rooted trees, with product given by grafting. For \(\tau,\sigma\in\mathcal{UT}\),
\[
    \tau \curvearrowright \sigma
    \coloneqq
    \sum_{v\in V(\sigma)}
    \tau\curvearrowright_v \sigma,
\]
where \(\tau\curvearrowright_v\sigma\) attaches the root of \(\tau\) to the vertex \(v\) of \(\sigma\). The Euclidean elementary differential map satisfies
\[
    F(\tomato{i})=V_i,
    \qquad
    F([\tau_1,\ldots,\tau_k]_i)
    =
    D^kV_i\bigl(F(\tau_1),\ldots,F(\tau_k)\bigr),
\]
and the grafting identity
\[
    F(\tau\curvearrowright\sigma)
    =
    F(\tau)\triangleright F(\sigma).
\]

\subsubsection{Branched rough paths in \(\mathbb{R}^n\): \gls{gl} as the cocommutative dual of \gls{bck}}
\label{app:gl_hopf}

It\^o integration introduces quadratic variation corrections, and word-indexed shuffle coordinates do not close under the It\^o chain rule. Branched rough paths restore closure by indexing the expansion by rooted trees \citep{hairer_geometric_2014}. The coordinate Hopf algebra is traditionally the noncocommutative \gls{bck} Hopf algebra, while the primitive log-ODE construction uses its cocommutative graded dual, the \gls{gl} Hopf algebra \citep{grossman_hopf-algebraic_1989}.

The \gls{gl} product is the Guin--Oudom product extending pre-Lie grafting. If \(\Delta_{\sqcup\sqcup}\alpha=\sum \alpha_{(1)}\otimes\alpha_{(2)}\) denotes the cocommutative unshuffle coproduct on forests, then
\[
    \alpha\star_{\mathrm{GL}}\beta
    =
    \sum
    \alpha_{(1)}
    \bigl(\alpha_{(2)}\curvearrowright \beta\bigr),
\]
with the standard extension of \(\curvearrowright\) from trees to forests \citep{oudom_lie_2004}. In particular, for trees \(\tau,\sigma\in\mathcal{UT}\),
\[
    \tau\star_{\mathrm{GL}}\sigma
    =
    \tau\sigma+\tau\curvearrowright\sigma.
\]
For example,
\[
    \tomato{m}\curvearrowright\ladderTwo{i}{j}
    =
    \ladderThree{i}{j}{m}
    +
    \cherry{i}{j}{m}.
\]
The Hopf algebra
\[
    \mathcal H_{\mathrm{GL}}
    =
    \bigl(S(\mathcal{UT}),\star_{\mathrm{GL}},\Delta_{\sqcup\sqcup}\bigr)
\]
is connected, graded, and cocommutative. Hence \cref{thm:milnor_moore} identifies it with the enveloping Hopf algebra of its primitive Lie algebra.

\begin{example}[Smooth branched signature]
Let \(X_t^a=t\) and \(X_t^b=t^2\) on \([0,1]\). The canonical smooth branched lift is defined recursively by
\[
    \langle \mathbf X_{0,t},\tomato{i}\rangle=X_t^i-X_0^i,
    \qquad
    \langle \mathbf X_{0,t},[\tau_1,\ldots,\tau_k]_i\rangle
    =
    \int_0^t
    \prod_{r=1}^k
    \langle \mathbf X_{0,s},\tau_r\rangle
    \,dX_s^i .
\]
Thus
\[
    \langle \mathbf X,\ladderTwo{b}{a}\rangle
    =
    \int_0^1 s\,d(s^2)
    =
    \frac{2}{3},
    \qquad
    \langle \mathbf X,\ladderTwo{a}{b}\rangle
    =
    \int_0^1 s^2\,ds
    =
    \frac{1}{3},
\]
and, since the non-planar forest product is commutative,
\[
    \langle \mathbf X,\cherry{b}{a}{b}\rangle
    =
    \int_0^1 s\,s^2\,d(s^2)
    =
    \frac{2}{5}.
\]
With \(V_a=\partial_y\) and \(V_b=y^2\partial_y\) on \(\mathbb R\),
\[
    F\left(\, \ladderTwo{b}{a} \right)
    =
    DV_b[V_a]
    =
    2y\partial_y,
    \qquad
    F\left(\, \cherry{b}{a}{b}\right)
    =
    D^2V_b[V_a,V_b]
    =
    2y^2\partial_y .
\]
This is a smooth, geometric special case of a branched lift. The raw tree integrals pair with elementary differentials using the symmetry convention in \cref{app:duality_dictionary}, with the bracket correction included for an It\^o lift.
\end{example}

\subsubsection{Post-Lie grafting and manifold elementary differentials}
\label{app:postlie_grafting}

On manifolds, the composition of differential operators is not represented by the Euclidean pre-Lie product alone. For a smooth connection \(\nabla\), the MKW elementary differentials used below are defined through total covariant derivatives. When \(\nabla\) has zero curvature and parallel torsion \(T\), this construction admits the post-Lie interpretation
\[
    U\triangleright V\coloneqq \nabla_UV,
    \qquad
    [U,V]_{\nabla}\coloneqq -T(U,V).
\]
The induced Lie bracket is
\[
    \llbracket U,V\rrbracket
    =
    U\triangleright V
    -
    V\triangleright U
    +
    [U,V]_{\nabla},
\]
which agrees with the Jacobi bracket of vector fields. Under the same flat/parallel-torsion assumptions, the post-Lie identities are
\[
    x\triangleright [y,z]_{\nabla}
    =
    [x\triangleright y,z]_{\nabla}
    +
    [y,x\triangleright z]_{\nabla},
\]
and
\[
    [x,y]_{\nabla}\triangleright z
    =
    x\triangleright (y\triangleright z)
    -
    (x\triangleright y)\triangleright z
    -
    y\triangleright (x\triangleright z)
    +
    (y\triangleright x)\triangleright z .
\]
Planar rooted trees give the free combinatorial model of the resulting post-Lie calculus. The relevant grafting operation is left grafting:
\[
    \tau \curvearrowright_{\ell} \sigma
    \coloneqq
    \sum_{v\in V(\sigma)}
    \tau\curvearrowright_{\ell,v}\sigma,
\]
where \(\tau\curvearrowright_{\ell,v}\sigma\) attaches the root of \(\tau\) to \(v\) as the left-most child. This convention records the order of covariant differentiations. At first order,
\[
    F\left(\, \ladderTwo{i}{j}\right)
    =
    V_j\triangleright V_i
    =
    \nabla_{V_j}V_i .
\]
For the manifold log-ODE construction, the essential algebraic property is the pseudo-bialgebra identity
\[
    \#(\alpha\star_{\mathrm{MKW}}\beta)
    =
    \#(\alpha)\circ\#(\beta),
\]
where \(\#\) is defined on ordered forests by total covariant derivatives. The ordered-forest construction is a pseudo-bialgebra map for any smooth connection \citep[Thm.~5.9]{kern_flow_2023}. For flat connections with parallel torsion, it additionally preserves the post-Lie structure.

\begin{example}[Second-order chain rule]
For one-node trees,
\[
    \tomato{j}\star_{\mathrm{MKW}}\tomato{i}
    =
    \tomato{j}\tomato{i}
    +
    \ladderTwo{i}{j}.
\]
Applying \(\#\) to a test function \(\varphi\in C^\infty(\mathcal M)\) gives
\[
    \#(\tomato{j}\star_{\mathrm{MKW}}\tomato{i})\varphi
    =
    \nabla^2\varphi(V_j,V_i)
    +
    (\nabla_{V_j}V_i)\varphi.
\]
By the covariant chain rule,
\[
    \nabla^2\varphi(V_j,V_i)
    +
    (\nabla_{V_j}V_i)\varphi
    =
    V_j(V_i\varphi).
\]
Thus the grafting term \(\ladderTwo{i}{j}\) is exactly the correction needed for \(\star_{\mathrm{MKW}}\) to represent composition of differential operators.
\end{example}

\subsubsection{Branched rough paths on \(\mathcal M\): the Munthe--Kaas--Wright Hopf algebra}
\label{app:mkw_hopf}

Let \(A\) denote the alphabet of driver labels, and let \(\operatorname{OF}(A)\) be the vector space spanned by \(A\)-decorated ordered forests. We write \(\tau\sigma\) for ordered concatenation of forests and
\[
    [\tau_1,\ldots,\tau_k]_i
\]
for the planar tree with root label \(i\) and ordered children \(\tau_1,\ldots,\tau_k\). Thus
\(
    \cherry{i}{j}{k}\neq \cherry{i}{k}{j}.
\)

The MKW coordinate algebra has the shuffle product of ordered forests and the full left-admissible-cut coproduct. Its graded dual $\hopfmkw$ has the cocommutative unshuffle coproduct
\[
\Delta(\tau_1\cdots\tau_k)
=\sum_{I\subseteq\{1,\ldots,k\}}\tau_I\otimes\tau_{I^c},
\]
where each subword retains its original order. The associated product is the planar Grossman--Larson product described below.

The primitives form the free Lie algebra on planar trees under $[u,v]_{\mathrm{conc}}=uv-vu$. Together with left grafting, they form the free post-Lie algebra on the driver labels, whose derived bracket is
\[
[u,v]_\star=u\star v-v\star u
=[u,v]_{\mathrm{conc}}+u\curvearrowright_\ell v-v\curvearrowright_\ell u.
\]
The pseudo-bialgebra map sends $[u,v]_\star$ to the Jacobi bracket of vector fields. Forest combinations such as $\tomato{i}\tomato{j}-\tomato{j}\tomato{i}$ are also primitive.

The product on \(\hopfmkw\) is the Guin--Oudom product associated with left grafting. If
\[
    \Delta_{\mathrm{unsh}}\alpha
    =
    \sum \alpha_{(1)}\otimes\alpha_{(2)}
\]
denotes the unshuffle coproduct, then
\[
    \alpha\star_{\mathrm{MKW}}\beta
    =
    \sum
    \alpha_{(1)}
    \bigl(
    \alpha_{(2)}\curvearrowright_{\ell}\beta
    \bigr),
\]
with the standard post-Lie extension of \(\curvearrowright_{\ell}\) from trees to ordered forests. In particular, for single trees \(\tau,\sigma\),
\[
    \tau\star_{\mathrm{MKW}}\sigma
    =
    \tau\sigma+\tau\curvearrowright_{\ell}\sigma.
\]
For example,
\[
    \tomato{m}\curvearrowright_{\ell}\ladderTwo{i}{j}
    =
    \ladderThree{i}{j}{m}
    +
    \cherry{i}{m}{j}.
\]
The second term records grafting at the root of \(\ladderTwo{i}{j}\), where the new child is inserted as the left-most child. Hence the planar order \(\cherry{i}{m}{j}\) is distinguished from \(\cherry{i}{j}{m}\).

\begin{example}[Smooth planarly branched signature]
Let \(X_t^a=t\) and \(X_t^b=t^2\) on \([0,1]\). For the smooth, ordered lift,
\[
    \langle \mathbf X_{0,t},\tomato{a}\tomato{b}\rangle
    =
    \int_{0<u<v<t}dX_u^a\,dX_v^b
    =
    \frac{2}{3}t^3,
\]
whereas
\[
    \langle \mathbf X_{0,t},\tomato{b}\tomato{a}\rangle
    =
    \int_{0<u<v<t}dX_u^b\,dX_v^a
    =
    \frac{1}{3}t^3.
\]
Consequently,
\[
    \langle \mathbf X,\cherry{b}{a}{b}\rangle
    =
    \int_0^1
    \langle \mathbf X_{0,t},\tomato{a}\tomato{b}\rangle
    \,dX_t^b
    =
    \frac{4}{15},
\]
while
\[
    \langle \mathbf X,\cherry{b}{b}{a}\rangle
    =
    \int_0^1
    \langle \mathbf X_{0,t},\tomato{b}\tomato{a}\rangle
    \,dX_t^b
    =
    \frac{2}{15}.
\]
Thus the ordered trees are distinguished by the signature. Their sum recovers
the corresponding unordered coefficient \(2/5\).
\end{example}

\begin{example}[Ordered grafting and the covariant chain rule]
Let $U=F(\tau)$, $V=F(\sigma)$, and $W=W_i$. The left-grafting recursion includes the two contributions
\[
(\nabla^2W)(U,V)+\nabla_{\nabla_UV}W
=\nabla_U(\nabla_VW).
\]
The order in $(\nabla^2W)(U,V)$ is determined by left grafting. Reversing it changes the result by
\[
(\nabla^2W)(U,V)-(\nabla^2W)(V,U)
=R(U,V)W-\nabla_{T(U,V)}W.
\]
The curvature and torsion therefore determine the contribution of reversed child order.

\end{example}

\clearpage
\section{Further Mathematical Details}
\label{appx:mathematical_details}

\subsection{Primitive bases and their dimensions}
\label{app:primitive_counts}
Write $a_n(d)$ for the number of non-planar $d$-coloured rooted trees. The multiset decomposition at the root gives the formal identity
\[
A_d(z):=\sum_{n\ge1}a_n(d)z^n
=dz\prod_{n\ge1}(1-z^n)^{-a_n(d)}.
\]
In particular $a_1=d$, $a_2=d^2$, and $a_3=d^3+d\binom{d+1}{2}$.

For planar trees, set $P_d(z)=\sum_{n\ge1}C_{n-1}d^nz^n$. Since ordered forests form the tensor algebra on planar trees, the dimensions $\ell_n(d)$ of its primitive space satisfy
\[
\frac1{1-P_d(z)}=\prod_{n\ge1}(1-z^n)^{-\ell_n(d)}.
\]
Consequently
\[
\ell_1=d,\qquad \ell_2=d^2+\binom d2,\qquad
\ell_3=2d^3+d^3+\frac{d^3-d}{3}.
\]
The degree-three summands count trees, brackets of degree-one and degree-two trees, and degree-three Lie words in single-node trees, respectively. A weighted Lyndon basis over the planar-tree alphabet gives a concrete primitive basis. The concatenation and derived brackets are related as in \cref{app:mkw_hopf}.

\subsection{Forest evaluation and the depth-two reduction}
\label{app:forest_lift}
For an ordered forest $f=\tau_1\cdots\tau_k$, define
\begin{equation}
\label{eq:forest_lift}
F(f)\phi=(\nabla^k\phi)(F(\tau_1),\ldots,F(\tau_k)),
\qquad F(1)\phi=\phi.
\end{equation}
Together with \cref{eq:elementary_differentials}, this specifies the map on the whole ordered-forest basis. Its linear extension is a pseudo-bialgebra map for a smooth connection \citep[Thm.~5.9]{kern_flow_2023}. In the flat, torsion-free case the derivatives are symmetric and the construction descends to the non-planar GL algebra. The coproduct identity then implies that primitives act as vector fields, as shown in \cref{appx:prop_primitive_vector_field}.

\begin{proposition}[Depth-two torsion-free reduction]
\label{prop:planar_degree_two_reduction}
The primitives of the MKW dual through degree two are spanned by $\tomato{i}$, $\ladderTwo{j}{i}$, and $\tomato{i}\tomato{j}-\tomato{j}\tomato{i}$ for $i<j$. Under \cref{eq:forest_lift},
\[
F(\tomato{i}\tomato{j}-\tomato{j}\tomato{i})=-T(W_i,W_j).
\]
Hence, if the connection is torsion-free, only the $d+d^2$ tree fields need to be evaluated at this depth, although the full primitive dimension remains $d+d^2+\binom d2$.
\end{proposition}
\begin{proof}
The unshuffle coproduct makes each tree and each displayed concatenation commutator primitive. The count follows from $\ell_1$ and $\ell_2$ above. For any smooth $\phi$,
\[
(\nabla^2\phi)(U,V)-(\nabla^2\phi)(V,U)
=[U,V]\phi-(\nabla_UV-\nabla_VU)\phi=-T(U,V)\phi.
\]

\end{proof}

\subsection{A depth-two It\^o log-ODE step}
\label{app:depth_two_ito}
Let $X$ be a continuous semimartingale, and write $x^i=X^i_{s,t}$, $Q^{ij}=[X^i,X^j]_{s,t}$ and $I^{ij}=\int_s^t X^i_{s,u}\,dX^j_u$. Repeated driver indices are summed from $1$ to $d$. It\^o's product rule gives $I^{ij}+I^{ji}=x^ix^j-Q^{ij}$, so the degree-two GL logarithm and its Euclidean vector-field lift are
\[
\lambda=x^i\tomato{i}+\lambda^{ij}\ladderTwo{j}{i},
\qquad \lambda^{ij}=I^{ij}-\tfrac12x^ix^j,
\qquad L=x^iW_i+\lambda^{ij}DW_j[W_i].
\]
For scalar Brownian motion with increment $w=B_{s,t}$ and $h=t-s$, this reduces to $L=wW-\tfrac h2DW[W]$. Taking $W(y)=y$ yields the time-one flow $y\mapsto y\exp(w-h/2)$, recovering the It\^o stochastic exponential.

On a torsion-free manifold, replace $DW_j[W_i]$ by $\nabla_{W_i}W_j$. For a general connection, the ordered-forest components contribute through \cref{eq:forest_lift}. The frame realisation used in the experiments is given in \cref{app:geometry_conventions}.

\paragraph{Discrete enhancement.}
For fine increments $\delta x_k$ and bracket increments $q_k^{ij}$, set $I_k^{ij}=\tfrac12(\delta x_k^i\delta x_k^j-q_k^{ij})$ and concatenate the local increments by Chen's identity. In the MKW dual the corresponding degree-two increment is
\[
g_k=\exp_\star\!\left(\delta x_k^i\tomato{i}-\tfrac12q_k^{ij}\ladderTwo{j}{i}\right).
\]
This is the degree-two correction supplied to the signature routine. For a smooth vector-field system driven by an $\alpha$-H\"older rough path with $1/3<\alpha<1/2$, the exact log-ODE flow has local error $O(|t-s|^{3\alpha})$ on bounded sets \cite{kern_flow_2023}.

\subsection{Vector Field Lift Pseudocode}
\label{appx:vf_lift_pseudocode}
In this section, we provide pseudocode for the vector field lifts of
\cref{subsec:neural_lift}. For vector fields $G,U_1,\dots,U_m$, write
\[
  \mathrm{TCD}_\nabla(G;U_1,\dots,U_m)
  :=
  (\nabla^m G)(U_1,\dots,U_m),
\]
where the total covariant derivative is given recursively by
\[
  \begin{aligned}
  (\nabla^0 G)(y) &= G(y),\\
  (\nabla^m G)(U_1,\dots,U_m)(y)
    &= \nabla_{U_1}\big((\nabla^{m-1}G)(U_2,\dots,U_m)\big)(y)\\
    &\quad
    - \sum_{r=2}^{m}
      (\nabla^{m-1}G)(U_2,\dots,\nabla_{U_1}U_r,\dots,U_m)(y).
  \end{aligned}
\]

\begin{minipage}[t]{0.49\linewidth}
\begin{algorithm}[H]
  \caption{$\hopfshuf$ Vector Field Lift}
  \label{alg:lyndon_lift_pseudocode}
  \begin{algorithmic}
\STATE {\bfseries Input:} driver vector fields $\{W_i\in\Gamma(T\mathcal M)\}_{i=0}^{d-1}$, represented in a chosen frame by coefficient functions $w_i:\mathcal M\to\mathbb R^n$, Hopf algebra $\mathcal H=\hopfshuf(d,N)$ with cached Lyndon basis $B$, geometry $(\mathcal M,\nabla)$
\STATE {\bfseries Precompute (once):} Lyndon words up to depth $N$; for each length-one basis index $k$, store its letter $r_k$; for each non-letter basis index $k$, store child indices $C_k=(p,q)$ from the standard factorisation
\STATE Initialise empty cache $\{F_k\}_{k=0}^{|B|-1}$

\STATE {\bfseries Memoised builder:} $\operatorname{Build}(k)$ returns cached $F_k$ if present
\STATE $C_k \gets B.\mathrm{children}[k]$
\IF{$C_k=()$}
  \STATE $i \gets r_k$
  \STATE $F_k \gets W_i$
\ELSE
  \STATE $(p,q)\gets C_k$
  \STATE $F_p\gets\operatorname{Build}(p),\;\; F_q\gets\operatorname{Build}(q)$
  \STATE $F_k\gets [F_p,F_q]_{\mathrm{Jac}}$
  \STATE \hspace{1em} equivalently, if using torsion \(T\), $F_k=\nabla_{F_p}F_q-\nabla_{F_q}F_p-T(F_p,F_q)$
\ENDIF
    \STATE Cache and return $F_k$

    \FOR{$k=0$ {\bfseries to} $|B|-1$}
      \STATE $F_k\gets\operatorname{Build}(k)$
    \ENDFOR

    \STATE {\bfseries Output:} lifted fields $\{F_k\}_{k=0}^{|B|-1}$ aligned with the truncated Lyndon basis of $\hopfshuf$
  \end{algorithmic}
\end{algorithm}
\end{minipage}
\hfill
\begin{minipage}[t]{0.49\linewidth}
\begin{algorithm}[H]
  \caption{Tree component of the $\hopfgl$, $\hopfmkw$ lift}
  \label{alg:tree_lift_pseudocode}
  \begin{algorithmic}
    \STATE {\bfseries Input:} driver vector fields $\{W_i\in\Gamma(T\mathcal M)\}_{i=0}^{d-1}$, represented in a chosen frame by coefficient functions $w_i:\mathcal M\to\mathbb R^n$, Hopf algebra $\mathcal H\in\{\hopfgl(d,N),\hopfmkw(d,N)\}$ with cached rooted-tree basis $B$, geometry $(\mathcal M,\nabla)$
    \STATE {\bfseries Convention:} for $\hopfgl$, use the Euclidean or flat torsion-free setting so that unordered children give symmetric elementary differentials; for $\hopfmkw$, use ordered children and total covariant derivatives for the chosen connection \(\nabla\)
    \STATE {\bfseries Precompute (once):} enumerate canonical non-planar rooted trees for $\hopfgl$ or planar rooted trees for $\hopfmkw$ up to depth $N$; for each basis index $k$, store root colour $r_k$ and child indices $C_k$
    \STATE Initialise empty cache $\{F_k\}_{k=0}^{|B|-1}$

    \STATE {\bfseries Memoised builder:} $\operatorname{Build}(k)$ returns cached $F_k$ if present
    \STATE $C_k \gets B.\mathrm{children}[k]$, \quad $i\gets r_k$
    \IF{$C_k=()$}
      \STATE $F_k\gets W_i$
    \ELSE
      \STATE $(p_1,\dots,p_m)\gets C_k$
      \STATE $F_{p_j}\gets\operatorname{Build}(p_j)$ for $j=1,\dots,m$
      \IF{$\mathcal H=\hopfgl$}
        \STATE $F_k(y)\gets D^m w_i(y)[F_{p_1}(y),\dots,F_{p_m}(y)]$
      \ELSE
        \STATE $F_k(y)\gets\mathrm{TCD}_\nabla(W_i;F_{p_1},\dots,F_{p_m})(y)$
      \ENDIF
    \ENDIF
    \STATE Cache and return $F_k$

    \FOR{$k=0$ {\bfseries to} $|B|-1$}
      \STATE $F_k\gets\operatorname{Build}(k)$
    \ENDFOR

    \STATE {\bfseries Output:} tree fields $\{F_k\}$. For $\hopfmkw$, evaluate the ordered-forest terms using \cref{eq:forest_lift} and the frame realisation in \cref{app:geometry_conventions}.
  \end{algorithmic}
\end{algorithm}
\end{minipage}

\newpage

\subsection{Finite-grid role of bracket channels}
\label{app:finite_grid_bracket_kernel}

At a fixed observation grid, geometric kernels act on path values, while branched kernels also depend on the supplied bracket. We describe this distinction and bound the effect of bracket reconstruction on the kernel features.

Let $\pi_n=\{0=t_0<t_1<\cdots<t_{N_n}=T\}$ be the observation grid. Define
$\Omega_n:=(\mathbb R^d)^{N_n+1}$,
$\mathcal A_n:=(\mathrm{Sym}(d))^{N_n+1}$, and
$\bar\Omega_n:=\Omega_n\times\mathcal A_n$. We write
$\bar x=(x,a)\in\bar\Omega_n$, where
$x=(x_{t_0},\ldots,x_{t_{N_n}})$ are path values and
$a=(a_{t_0},\ldots,a_{t_{N_n}})$ are bracket values. For a continuous semimartingale
$X$, the supplied enhanced observation is $(X_{\pi_n},A_{\pi_n})$, where
$A_t:=\langle X\rangle_t$. Let $\Pi:\bar\Omega_n\to\Omega_n$ be the projection
$\Pi(x,a)=x$. Throughout, $|\cdot|$ denotes the Euclidean norm on
$\mathbb R^d$ and the Frobenius norm on $\mathrm{Sym}(d)$.

Assume throughout this subsection that feature norms and bracket statistics are integrable under the laws considered. For a kernel $k(u,v)=\langle\varphi(u),\varphi(v)\rangle_{\mathcal K}$, write
\[
    \mathrm{MMD}_k(P,Q)
    :=
    \left\|
        \mathbb E_{U\sim P}\varphi(U)
        -
        \mathbb E_{V\sim Q}\varphi(V)
    \right\|_{\mathcal K}.
\]
In the training objective, we omit the data--data term in the squared MMD
since it is constant in $\theta$.

\begin{proposition}[Projected geometric kernels and bracket visibility]
\label{prop:projected_geo_and_branched_visibility}
Let $k_{\mathrm{geo}}^N$ be a finite-depth geometric signature kernel computed
from the piecewise-linear interpolation of $x\in\Omega_n$, with feature map
$\varphi_{\mathrm{geo}}^N:\Omega_n\to\mathcal H_{\mathrm{geo}}^N$. Extend it
to enhanced observations by ignoring brackets,
$\bar k_{\mathrm{geo}}^N((x,a),(y,b)):=k_{\mathrm{geo}}^N(x,y)$. Then, for
any enhanced laws $P,Q\in\mathcal P(\bar\Omega_n)$,
\[
    \mathrm{MMD}_{\bar k_{\mathrm{geo}}^N}(P,Q)
    =
    \mathrm{MMD}_{k_{\mathrm{geo}}^N}(\Pi_\#P,\Pi_\#Q).
\]
In particular, if $\Pi_\#P=\Pi_\#Q$, then
$\mathrm{MMD}_{\bar k_{\mathrm{geo}}^N}(P,Q)=0$.

Let $\varphi_{\mathrm{br}}^N:\bar\Omega_n\to\mathcal H_{\mathrm{br}}^N$ be
the truncated branched feature map, with kernel $k_{\mathrm{br}}^N(\bar x,\bar y)=\langle\varphi_{\mathrm{br}}^N(\bar x),\varphi_{\mathrm{br}}^N(\bar y)\rangle$. Suppose a scalar bracket statistic has the form
\[
b(x,a)=r(x)+\langle\varphi_{\mathrm{br}}^N(x,a),v_b\rangle.
\]
If $\Pi_\#P=\Pi_\#Q$, then
\[
|\mathbb E_Pb-\mathbb E_Qb|
\le \|v_b\|\,\mathrm{MMD}_{k_{\mathrm{br}}^N}(P,Q).
\]
Thus laws with the same path projection and different bracket means have zero geometric MMD and strictly positive branched MMD.

\end{proposition}

\begin{proof}
The geometric feature map factors through $\Pi$, giving the first identity. For the branched features, equality of the projected laws implies $\mathbb E_Pr=\mathbb E_Qr$. Writing $\mu_P=\mathbb E_P\varphi_{\mathrm{br}}^N$ and $\mu_Q=\mathbb E_Q\varphi_{\mathrm{br}}^N$ gives
\[
\mathbb E_Pb-\mathbb E_Qb=\langle\mu_P-\mu_Q,v_b\rangle,
\]
and the result follows from Cauchy's inequality.
\end{proof}

For the raw tree coordinates used in the implementation at $N\ge2$, take $b=Q^{ij}_{0,T}$ and $r(x)=x_{0,T}^ix_{0,T}^j$. The identity
\[
Q^{ij}_{0,T}=x_{0,T}^ix_{0,T}^j-I^{ij}_{0,T}-I^{ji}_{0,T}
\]
gives $\|v_b\|=\sqrt2$ for $i\ne j$ and $\|v_b\|=2$ for $i=j$. Here zero geometric MMD suffices, since $r$ is a linear combination of geometric level-two coordinates. When the full forest coordinates are retained, $x(\tomato{i}\tomato{j})$ represents $r(x)$ as a feature coordinate as well.

We next compare supplied brackets with brackets reconstructed from grid samples.

Let $R_n:\Omega_n\to\mathcal A_n$ be a bracket estimator. The realised
quadratic covariation estimator is the canonical example:
\[
    [R_n(x)]_{t_m}
    :=
    \sum_{k=0}^{m-1}
    (x_{t_{k+1}}-x_{t_k})(x_{t_{k+1}}-x_{t_k})^\top,
    \qquad m=0,\ldots,N_n.
\]
Define the supplied and reconstructed enhanced laws by
$\bar P_n^{\mathrm{sup}}:=\mathrm{Law}(X_{\pi_n},A_{\pi_n})$ and
$\bar P_n^{\mathrm{rec}}:=\mathrm{Law}(X_{\pi_n},R_n(X_{\pi_n}))$.

\begin{proposition}[Cost of reconstructing bracket coordinates]
\label{prop:bracket_reconstruction_cost}
Equip $\bar\Omega_n$ with
$d_{\bar\Omega_n}((x,a),(y,b))
:=\max_m|x_{t_m}-y_{t_m}|+\max_m|a_{t_m}-b_{t_m}|$. Let
$\mathcal D_n\subseteq\bar\Omega_n$ contain the supports of
$\bar P_n^{\mathrm{sup}}$ and $\bar P_n^{\mathrm{rec}}$. If
$f:\bar\Omega_n\to\mathbb R$ is $L_f$-Lipschitz on $\mathcal D_n$ and integrable under both laws, then
\[
    \left|
        \mathbb E f(X_{\pi_n},R_n(X_{\pi_n}))
        -
        \mathbb E f(X_{\pi_n},A_{\pi_n})
    \right|
    \le
    L_f\,
    \mathbb E\max_m
    \left|R_n(X_{\pi_n})_{t_m}-A_{t_m}\right|.
\]
If $\varphi_{\mathrm{br}}^N$ is $L_N$-Lipschitz on $\mathcal D_n$ with respect
to $d_{\bar\Omega_n}$, then
\[
    \mathrm{MMD}_{k_{\mathrm{br}}^N}
    (\bar P_n^{\mathrm{rec}},\bar P_n^{\mathrm{sup}})
    \le
    L_N\,
    \mathbb E\max_m
    \left|R_n(X_{\pi_n})_{t_m}-A_{t_m}\right|.
\]
\end{proposition}

\begin{proof}
Couple the supplied and reconstructed enhanced observations by using the same
grid path $X_{\pi_n}$. Since their path coordinates are identical,
\[
    d_{\bar\Omega_n}
    \bigl((X_{\pi_n},R_n(X_{\pi_n})),(X_{\pi_n},A_{\pi_n})\bigr)
    =
    \max_m\left|R_n(X_{\pi_n})_{t_m}-A_{t_m}\right|.
\]
The two coupled observations lie in $\mathcal D_n$ almost surely, so the
Lipschitz bound for $f$ gives the first claim after taking expectations.

For the MMD bound, Jensen's inequality and Lipschitzness of
$\varphi_{\mathrm{br}}^N$ on $\mathcal D_n$ give
\[
\begin{aligned}
    \mathrm{MMD}_{k_{\mathrm{br}}^N}
    (\bar P_n^{\mathrm{rec}},\bar P_n^{\mathrm{sup}})
    &\le
    \mathbb E
    \left\|
        \varphi_{\mathrm{br}}^N(X_{\pi_n},R_n(X_{\pi_n}))
        -
        \varphi_{\mathrm{br}}^N(X_{\pi_n},A_{\pi_n})
    \right\| \\
    &\le
    L_N\,
    \mathbb E\max_m
    \left|R_n(X_{\pi_n})_{t_m}-A_{t_m}\right|.
\end{aligned}
\]
\end{proof}

\begin{remark}[Bracket reconstruction]
\label{rem:reconstruction_not_geo_baseline}
The law $\bar P_n^{\mathrm{rec}}$ augments grid samples with a reconstructed bracket, whereas the geometric baseline uses only the path projection. Proposition~\ref{prop:bracket_reconstruction_cost} bounds the resulting feature discrepancy by the bracket-estimation error.

\end{remark}

\begin{remark}[Downstream marginal errors]
\label{rem:scope_and_downstream_marginal_errors}
At a fixed observation grid, the geometric kernel acts on the projected law $\Pi_\#P$, while the branched kernel acts on the enhanced law $P$ on $\bar\Omega_n$.

The same reconstruction term can propagate to downstream finite-dimensional
statistics. If a numerical solver or observation map
\(F_n:\bar\Omega_n\to C([0,T];\mathbb R^m)\) is Lipschitz on \(\mathcal D_n\),
then for any Lipschitz statistic
\(\Psi:C([0,T];\mathbb R^m)\to\mathbb R^r\) and any unit linear functional
\(\ell:\mathbb R^r\to\mathbb R\), the function
\(f=\ell\circ\Psi\circ F_n\) is Lipschitz on the same set. Proposition
\ref{prop:bracket_reconstruction_cost} then gives an elementary finite-grid
propagation bound for such downstream statistics.
\end{remark}

\begin{proposition}[Primitives lift to vector fields]\label[proposition]{appx:prop_primitive_vector_field}
Let $\mathcal H$ be a bialgebra and let $F:\mathcal H\to\mathrm{Diff}(M)$ be a pseudo-bialgebra map in the
sense of \cref{def:pseudo_bialgebra_map}. If $p\in\mathcal H$ is primitive, i.e.
$\Delta p=p\otimes 1+1\otimes p$, then $F(p)\in\Gamma(TM)$.
\end{proposition}

\begin{proof}\label{appx:proof_leibniz_derivation}
It suffices to show that $F(p)$ is a derivation of \(C^\infty(M)\).
Fix arbitrary $\phi,\psi\in C^\infty(M)$. By property (b) of
\cref{def:pseudo_bialgebra_map},
\[
    F(p)(\phi\cdot\psi)
    =
    F(\Delta p)(\psi\otimes\phi),
\]
where, for simple tensors, \(F(a\otimes b)(\psi\otimes\phi)\) denotes
\((F(a)\psi)\cdot(F(b)\phi)\). Since \(p\) is primitive,
\[
\begin{aligned}
    F(\Delta p)(\psi\otimes\phi)
    &=
    F(p\otimes 1+1\otimes p)(\psi\otimes\phi)\\
    &=
    (F(p)\psi)\cdot F(1)\phi
    +
    F(1)\psi\cdot (F(p)\phi).
\end{aligned}
\]
The unit condition gives \(F(1)=\mathrm{Id}\). Hence
\[
\begin{aligned}
    F(p)(\phi\cdot\psi)
    &=
    (F(p)\psi)\cdot\phi+\psi\cdot(F(p)\phi)\\
    &=
    (F(p)\phi)\cdot\psi+\phi\cdot(F(p)\psi),
\end{aligned}
\]
where the last equality uses commutativity of pointwise multiplication. Thus
\(F(p)\) satisfies the Leibniz rule, and therefore \(F(p)\in\Gamma(TM)\).
\end{proof}

\section{Further Experiments}
\subsection{Recovering Log-NCDE}
\label{appx:recovering_log_ncde}
For $\mathcal H=\hopfshuf$ and $\mathcal M=\mathbb R^u$, B-NRDE recovers the log-NCDE vector field. On PPG-DaLiA \cite{Reiss2019}, both models attain a \gls{mse} of 0.134 after 10 epochs.

\subsection{Signature Computation Times}
We report signature precomputation times for the relevant pySigLib log signature computations below, in \cref{fig:signature_computation_time}. 
\begin{figure}[h]
    \centering

    \begin{subfigure}{0.32\textwidth}
        \centering
        \includegraphics[width=\linewidth]{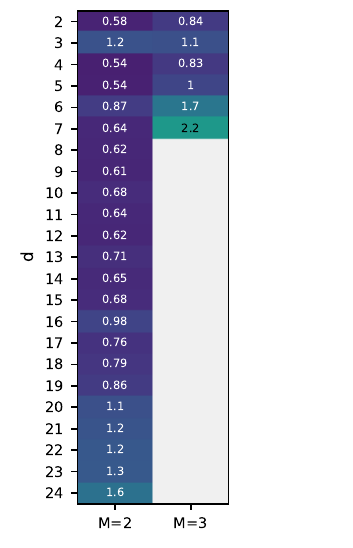}
        \caption{Planarly branched log-signature, $\hopfmkw$.}
        \label{fig:planar_logsig_comp_time}
    \end{subfigure}
    \hfill
    \begin{subfigure}{0.32\textwidth}
        \centering
        \includegraphics[width=\linewidth]{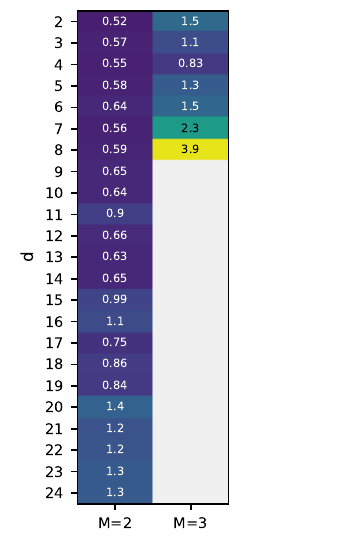}
        \caption{Branched log-signature, $\hopfgl$.}
        \label{fig:nonplanar_logsig_comp_time}
    \end{subfigure}
    \hfill
    \begin{subfigure}{0.32\textwidth}
        \centering
        \includegraphics[width=\linewidth]{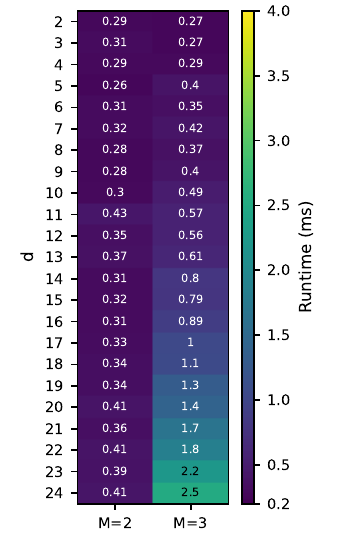}
        \caption{Log-signature, $\hopfshuf$.}
        \label{fig:logsig_comp_time}
    \end{subfigure}

    \caption{Milliseconds to compute log-signatures for $\mathcal{H} \in \{\hopfshuf, \hopfgl, \hopfmkw\}$ with path dimension $d \in [2, 24]$ and signature truncation depth $N \in \{2,3\}$. The custom CUDA kernels implemented in pySigLib do not support tree counts greater than 1024, producing the blank cells in \cref{fig:planar_logsig_comp_time} and \cref{fig:nonplanar_logsig_comp_time}.}
    \label{fig:signature_computation_time}
\end{figure}

\section{Experimental Details}
\label{app:experimental_details}
\subsection{Software and Hardware Details}
\label{appx:software_hardware_details}
\paragraph{Software details} All experiments were conducted on Python 3.13 using \href{https://github.com/jax-ml/jax}{JAX} 0.8.1 \cite{bradbury_jax_2018}. We use \href{https://github.com/patrick-kidger/diffrax}{Diffrax} 0.7.1 \cite{kidger_neural_2021} for its Euclidean differential equation framework, \href{https://github.com/patrick-kidger/equinox}{Equinox} 0.13.2 \cite{kidger_equinox_2021} as our neural network framework, \href{https://github.com/google-deepmind/optax}{Optax} 0.2.6 \cite{deepmind_deepmind_2020} for its implementation of the Muon optimizer \cite{jordanMuonOptimizerHidden2024}, and \href{https://github.com/smorad/cyreal}{Cyreal} 0.1.5 \cite{morad_smoradcyreal_2026} for dataloading. All $\log_\mathcal{H}$-signatures are formed using a \href{https://github.com/luke-a-thompson/pySigLibax}{custom PySigLib fork} \cite{shmelev_pysiglib_2025}, with the log-ODE method provided by \href{https://github.com/luke-a-thompson/roughrax}{Roughrax} and geometric numerical integrators provided by \href{https://github.com/luke-a-thompson/georax}{Georax} \cite{shmelev2026expliciteffectivelysymmetricschemes}.

\paragraph{Released packages} We release two JAX packages to support the present work: one for the log-ODE method on manifolds and one containing the benchmark datasets used.

\begin{table}[H]
    \centering
    \caption{New packages introduced to support the present work.}
    \label{tab:new_packages}
    \begin{tabular}{llll}
        \toprule
        \textbf{Package} & \textbf{Ecosystem} & \textbf{Functionality} & \textbf{License} \\
        \midrule
        \href{https://github.com/luke-a-thompson/Roughrax}{Roughrax}
        & JAX
        & Euclidean and manifold log-ODE
        & \texttt{Apache-2.0} \\
        \href{https://github.com/luke-a-thompson/RoughBench}{RoughBench}
        & JAX
        & Rough volatility (e.g., rBergomi) and SPD diffusion benchmarks
        & \texttt{Apache-2.0} \\
        \bottomrule
    \end{tabular}
\end{table}

\paragraph{Hardware details} All model training and evaluation was performed on an NVIDIA RTX 5080 (16GB) with CUDA 13.1, an AMD Ryzen 9 9950X3D processor running on Ubuntu 24.04 with 64GB of system memory.

\subsection{Rough Volatility}
\label{appx:exp_det_rough_volatility}
\Gls{rbergomi} is traditionally formulated as a system of coupled Volterra \glspl{sde} with rough kernels:
\begin{equation*}
    \begin{aligned}
        S_t &= S_0 + \int_0^t S_s \sqrt{v_s}\,dB_s,\\
        v_t &= \xi_0(t)\exp\left(\frac{\eta}{\Gamma\bigl(H+\tfrac12\bigr)}
        \int_0^t (t-s)^{H-\tfrac12}\,dW_s
        -\frac{\eta^2}{2\,\Gamma\bigl(H+\tfrac12\bigr)^2}
        \int_0^t (t-s)^{2H-1}\,ds\right).
    \end{aligned}
\end{equation*}
Here, \(S_t\) denotes the underlying asset price, \(v_t\) denotes the instantaneous variance process, \(\eta\) is the vol-of-vol, and \(H\ \in (0, \frac{1}{2})\) is the Hurst exponent controlling volatility roughness. The forward variance curve \(\xi_0(t)\) satisfies \(\mathbb{E}[v_t] = \xi_0(t)\) and the correlation of the two drivers \(W_t, B_t\) is given by \(\operatorname{Corr}(dW_t, dB_t) = \rho \in [-1, 1]\).

We simulate the following \gls{rde} parameterisation of \citet{bonesini_rough_2024}, with the forward variance curve specified below.
\begin{equation}
    \begin{aligned}
        \label{eq:rbergomi_rde_form}
        S_t &= S_0 + \int_0^t S_u \sqrt{v_0} \exp\left\{\nu V_u - \frac{\nu^2 u^{2H}}{2\Gamma(H+\frac{1}{2})^2}\right\} d\mathbf{W}_u
        - \int_0^t \frac{1}{2} S_u v_0 \exp\left\{2\nu V_u - \frac{\nu^2 u^{2H}}{\Gamma(H+\frac{1}{2})^2}\right\} du \\
        V_t &= \frac{1}{\Gamma(H+\frac{1}{2})} \int_0^t (t-u)^{H-\frac{1}{2}} \left( \rho dW_u + \sqrt{1-\rho^2} dB_u \right),
    \end{aligned}
\end{equation}
where $W$ and $B$ are independent Brownian motions and $\mathbf W$ denotes the Brownian component of their joint adapted rough enhancement with $V$ and \(V_t\) is a Riemann-Liouville volatility process computed following \cite{bennedsen_hybrid_2017}. We then solve \cref{eq:rbergomi_rde_form} as a Wong-Zakai \gls{ode} driven by a two-dimensional lead-lag path \((W_t, V_t)\), yielding convergence to the It\^o solution. We use the calibration results of \citet{callum_rough_2023} and set \(v_0 = 0.04,\ \nu = 1.991,\ H=0.25,\ \rho = -0.848\) which shows a \gls{mse} of \(3.73\times10^{-5}\) of model-generated implied volatilities over SPX compared to market implied volatilities on 30/05/2022.

For the displayed normalisation, $\operatorname{Var}(V_t)=t^{2H}/\{2H\Gamma(H+\tfrac12)^2\}$. The volatility in \cref{eq:rbergomi_rde_form} therefore corresponds to $\eta=2\nu$ and
\[
\xi_0(t)=v_0\exp\!\left\{\frac{\nu^2t^{2H}}{\Gamma(H+\tfrac12)^2}\left(\frac1H-1\right)\right\}.
\]
The Wong--Zakai discretisation linearly interpolates the Brownian path and the one-step delayed volatility driver. The latter takes values $(0,V_{t_0},\ldots,V_{t_{n-1}})$ on $(t_0,\ldots,t_n)$, and the resulting ODE is integrated with fixed-step Tsit5.

\subsection{Sim-to-Real Dynamics Forecasting}
\label{appx:exp_det_motion_dynamics}
In this section, we describe our implementation of the sim-to-real dynamics forecasting experiment developed by \citet{bastian_continuous-time_2025}.

We begin by generating the \texttt{FREE\_ROTATION} synthetic dataset available at their \href{https://github.com/bastianlb/forecasting-rotational-dynamics/tree/main}{official repository}. We then convert to \texttt{.npy} and convert the (quaternion-last) data into a timeseries of $\mathrm{SO}(3)$ rotation matrices and form stride-one length-twenty sliding windows across the time series and flatten the $\mathbb{R}^{3 \times 3}$ matrices into $\mathbb{R}^9$ vectors for feeding into the models. This produces a dataset of shape $\mathbb{R}^{\text{batch} \cdot \text{damping} \times \text{windows} \cdot 20 \times 9}$.

Each twenty-sample window is split into twelve reconstruction samples and eight prediction samples. The extrapolator is fitted to the reconstruction segment and supplies the continuation over the prediction segment, whose ground truth provides the training target. We use the following extrapolators.

\begin{enumerate}
    \item $\mathrm{SO}(3)$-\gls{ncde}: Hermite polynomial
    \item $\mathrm{SO}(3)$-GRU: Hermite polynomial
    \item SG-\gls{ncde}: Weighted Savitsky-Golay polynomial
    \item \Gls{bnrde}: \gls{mlp}
\end{enumerate}

To save parameters in the readout \gls{mlp}, the baseline models output vectors in $\mathbb{R}^6$, which are transformed back into $\mathbb{R}^{3 \times 3}$ rotation matrices via a Gram-Schmidt procedure \cite{zhou_continuity_2020}. \Gls{bnrde} evolves directly on $\mathrm{SO}(3)$ with an identity readout.

\subsection{Mean-reverting Dynamics over the SPD Manifold under Affine-invariant Geometry}
We simulate \cref{eq:affine_invariant_diffusion} on $\mathbb S^{++}_d$ with $d=3$ over horizon $T=1$ on a uniform grid of $N=1025$ time points, with initial condition $X_0 = I_3$.
We set the (matrix-valued) parameter
\[
\Sigma =
\begin{bmatrix}
0.18 & 0.04 & -0.02\\
0.02 & 0.14 & 0.06\\
0.00 & 0.08 & 0.16
\end{bmatrix},
\qquad
A =
\begin{bmatrix}
0.00 & 0.20 & 0.00\\
-0.10 & 0.00 & 0.10\\
0.00 & -0.15 & 0.00
\end{bmatrix},
\qquad
\gamma = 0.25,
\qquad
\varepsilon = 0.1,
\qquad
\sigma = 0.4.
\]
Define
\[
Q := \Sigma^\top \Sigma,\qquad
M := b := (d-1)Q + \varepsilon I_d,\qquad
H := -\gamma I_d + A.
\]
For the affine-invariant mean reversion we use the (PSD) stiffness
\[
K := \Pi_{\mathbb S^+_d}\bigl(\mathrm{sym}(-H)\bigr),
\qquad
\eta := \tfrac{1}{d}\mathrm{tr}(K),
\]
where $\Pi_{\mathbb S^+_d}$ denotes eigenvalue clipping to $\mathbb S^+_d$.
We take the symmetric Brownian driver $B_t\in\mathrm{Sym}(d)$ with covariation
\[
\mathrm d\langle B_{ab},B_{cd}\rangle_t
=\tfrac12\bigl(\delta_{ac}\delta_{bd}+\delta_{ad}\delta_{bc}\bigr)\,\mathrm dt,
\]
(corresponding to independent diagonal entries and off-diagonals scaled by $1/\sqrt{2}$), and use the identity correlation in the symmetric degrees of freedom.

\paragraph{Analytic quadratic variation.}
For the It\^o diffusion
\[
\mathrm d^\nabla X_t
=
\eta\,\log_{X_t}(M)\,\mathrm dt
+
\sigma\,X_t^{1/2}\,\mathrm dB_t\,X_t^{1/2},
\]
the entrywise quadratic covariation of $X_t$ is
\[
\mathrm d\langle X_{ij},X_{mn}\rangle_t
=
\frac{\sigma^2}{2}\Bigl(
X_{im}X_{jn}+X_{in}X_{jm}
\Bigr)\,\mathrm dt,
\qquad 1\le i,j,m,n\le d.
\]
Writing $x_t=\mathrm{vech}(X_t)\in\mathbb R^{q}$ with $q=\frac{d(d+1)}{2}$, the corresponding quadratic variation is obtained by the linear projection
\[
\mathrm d\langle x\rangle_t
=
E\,\mathrm d\langle \mathrm{vec}(X)\rangle_t\,E^\top,
\]
where $E\in\mathbb R^{q\times d^2}$ selects the lower-triangular coordinates consistent with our $\mathrm{vech}$ convention.
The simulator uses the affine-invariant exponential update
\[
X_{k+1}=\operatorname{Exp}_{X_k}\!\left(\eta\log_{X_k}(M)\,\Delta t+
\sigma X_k^{1/2}\Delta B_kX_k^{1/2}\right),
\]
whose continuous limit is \cref{eq:affine_invariant_diffusion} with the affine-invariant Levi-Civita connection. In ambient coordinates its drift is $\eta\log_X(M)+\sigma^2(d+1)X/4$. The per-step bracket is obtained by multiplying the displayed bracket density at $X_k$ by $\Delta t$.

\subsection{Frame realisation and geometric updates}
\label{app:geometry_conventions}
The neural vector-field lift uses the connection that makes the chosen frame parallel. Writing $W_i=E_a w_i^a$ and summing over repeated frame indices gives
\[
\nabla^{\mathrm{fr}}_{W_i}W_j=E_a\,Dw_j^a[W_i],
\qquad T(E_a,E_b)=-[E_a,E_b].
\]
Roughrax evaluates the coefficient derivatives by JVPs. At depth two, the ordered-forest primitive $\tomato{i}\tomato{j}-\tomato{j}\tomato{i}$ maps to $w_i^aw_j^b[E_a,E_b]$. For expanded forest log-coefficients $c_{ij}=-c_{ji}$, the implementation contracts $\tfrac12\sum_{i,j}c_{ij}w_i^aw_j^b[E_a,E_b]$, retaining the complete primitive contribution.

For SO(3), $E_A(R)=RA$ with $A^\top=-A$, and $[E_A,E_B]=E_{AB-BA}$. For SPD matrices, $E_A(X)=AX+XA$ with $A^\top=A$, and $[E_A,E_B](X)=X[A,B]-[A,B]X$. Tangent vectors are converted to symmetric frame coordinates by solving $AX+XA=U$. This frame connection is the one used by the neural lift, while the simulator above uses the affine-invariant Levi-Civita connection
\[
\nabla^{\mathrm{LC}}_UV=DV[U]-\tfrac12(UX^{-1}V+VX^{-1}U).
\]

The $\mathrm{CF\text{-}EES}(2,5)$ implementation selects an eighth-degree Taylor approximation $p_8$ to meet its fifth-order symmetry requirement. On SO(3), a sign-adjusted QR factorisation of $p_8(A)$ gives the rotation applied to the state. On SPD matrices, the update is $X\mapsto p_8(A)Xp_8(A)^\top$ for symmetric $A$. The positivity of the even Taylor polynomial on the real line makes this congruence positive definite. The order-two chart available to other solvers uses the Cayley transform on SO(3) and $p_2$ on SPD matrices.

The model evolves directly on the problem manifold, with a learned initial state and identity readout. The SO(3) initial state uses the six-dimensional Gram--Schmidt representation, while the SPD initial state is the matrix exponential of a learned symmetric matrix.

\subsection{Numerical Integration}
NCDE-type baselines are solved with Tsit5 \cite{tsitouras_rungekutta_2011} using automatic initial stepsizing as described in \citet[Sec.~2.4]{hairer_solving_2008} and PID step size control controller following \citet{soderlind_automatic_2002} and \citet[Sec.~4.2]{hairer_solving_2002}. Euclidean signature-based models, including log-NCDE, NRDE, and B-NRDE, in the rBergomi task use Heun with a fixed step size. Manifold B-NRDE uses the $\mathrm{CF\text{-}EES}(2,5)$ commutator-free integrator. For fixed-step log-ODE runs, each signature window uses one solver step over the normalised interval $[0,1]$. Backpropagation through the solver proceeds by the discretise-then-optimise method \cite{ma_comparison_2021}.

The reported SPD distance averages the univariate $W_1$ distances of the sorted eigenvalues over the three eigenvalue indices and evaluation batches. Reported uncertainties are sample standard deviations across evaluation replicates.

\subsection{Hyperparameters}
\label{appx:hyperparameters}

All models are trained for 100 epochs with batch size 1024, gradient clipping at 1.0, and a 0.8/0.1/0.1 train/validation/test split. We use Muon with learning rate $5 \times 10^{-4}$, $(\beta_1,\beta_2)=(0.9,0.999)$, weight decay $1 \times 10^{-6}$, and $\epsilon = 1 \times 10^{-8}$. PID tolerances $\mathrm{rtol} = 1 \times 10^{-3}$ and $\mathrm{atol} = 1 \times 10^{-3}$, with minimum step size $1 \times 10^{-4}$. Unless otherwise stated, models use initial-condition and vector-field MLPs of width 32, depth 2, and latent state dimension 32. Signature depth is always taken at $N=2$. Hyperparameters differing between models are listed in \cref{tab:atom_hparams_diff}.
\begin{table}[H]
\centering
\caption{Hyperparameters that vary across experiments/models.}
\label{tab:atom_hparams_diff}
\setlength{\tabcolsep}{5pt}
\renewcommand{\arraystretch}{1.15}
\resizebox{\linewidth}{!}{
\begin{tblr}{
  colspec = {l l c c c c c c c c},
  row{1}  = {font=\bfseries, c},
  row{2}  = {font=\bfseries, c},
  cell{1}{1} = {r=2}{c},
  cell{1}{2} = {r=2}{c},
  cell{1}{3} = {c=8}{c},
  vline{3} = {-}{},
  hline{1,Z}       = {-}{},        
  hline{2}         = {3-10}{},     
  hline{3,9,16}    = {-}{},        
}
Experiment & Hyperparameter & Model &  &  &  &  &  &  &  \\
           &                & GRU & xLSTM & Stacked xLSTM & \gls{ncde} & \gls{ncde}++ & log-\gls{ncde} & \gls{nrde} & B-NRDE \\
\SetCell[r=6]{c} {Rough\\Bergomi}
  & Activation              & --- & --- & --- & Softplus  & Softplus  & Softplus    & Softplus    & LipSwish \\
  & Final activation        & --- & --- & --- & Identity  & $\tanh$   & $\tanh$     & $\tanh$     & $\tanh$ \\
  & Control interpolation   & --- & --- & --- & Linear    & Hermite   & Linear      & Linear      & Linear \\
  & Signature window        & --- & --- & --- & ---       & ---       & 16          & 16          & 16 \\
  & Hopf algebra            & --- & --- & --- & ---       & ---       & $\hopfshuf$ & $\hopfshuf$ & $\hopfgl$ \\
  & Signature kernel        & --- & --- & --- & Geometric & Geometric & Geometric   & Geometric   & Branched \\
\SetCell[r=7]{c} {$\mathrm{SO}(3)$\\Dynamics}
  & Activation              & Softplus & --- & --- & Softplus & Softplus & Softplus    & Softplus    & LipSwish \\
  & Final activation        & ---      & --- & --- & ---      & ---      & ---         & ---         & --- \\
  & Control interpolation   & ---      & --- & --- & Hermite  & Hermite  & Hermite     & Hermite     & --- \\
  & Signature window        & ---      & --- & --- & ---      & ---      & 4           & 4           & 4 \\
  & Hopf algebra            & ---      & --- & --- & ---      & ---      & $\hopfshuf$ & $\hopfshuf$ & $\hopfshuf$ \\
  & Signature kernel        & ---      & --- & --- & ---      & ---      & ---         & ---         & --- \\
  & Extrapolator            & ---      & --- & --- & Hermite  & Hermite, \gls{sg} & Hermite & Hermite & \gls{mlp} \\
\SetCell[r=6]{c} {\gls{spdman}\\Dynamics}
  & Activation              & --- & --- & --- & ---       & ---       & ---         & ---         & --- \\
  & Final activation        & --- & --- & --- & ---       & ---       & ---         & ---         & --- \\
  & Control interpolation   & --- & --- & --- & ---       & ---       & ---         & ---         & --- \\
  & Signature window        & --- & --- & --- & ---       & ---       & 64          & 64          & 64 \\
  & Hopf algebra            & --- & --- & --- & ---       & ---       & $\hopfshuf$ & $\hopfshuf$ & $\hopfmkw$ \\
  & Signature kernel        & --- & --- & --- & ---       & ---       & Geometric   & Geometric   & Branched \\
\end{tblr}
}
\end{table}


\end{document}